\documentclass[twocolumn,natbib]{svjour3}

\makeatletter

\let\proof\@undefined
\let\endproof\@undefined
\makeatother

\usepackage{amsmath,amssymb,amsthm}

\usepackage{graphicx}
\usepackage{subfig}
\usepackage{color}
\usepackage{hyperref}
\usepackage{siunitx}
\usepackage{macros}
\usepackage{paralist}
\usepackage{multirow}

\usepackage[utf8]{inputenc}
\usepackage[T1]{fontenc}
\usepackage{caption}
\usepackage{flushend} 

\usepackage[ruled,vlined,linesnumbered]{algorithm2e}
\SetInd{0.1em}{0.6em}
\SetCommentSty{textit}
\SetAlgoInsideSkip{smallskip}
\DontPrintSemicolon

\newcommand\mySubref[1]{\protect\subref{#1}}
 
\newcommand\varFont[1]{\textsf{#1}}
\newcommand\varFontNoSize[1]{\textsf{#1}}
\newcommand\valueFont[1]{\texttt{#1}}
\newcommand\valueFontNoSize[1]{\texttt{#1}}

\newcommand\varFontNoSizeSmall[1]{\varFontNoSize{#1}}
\newcommand\valueFontNoSizeSmall[1]{\valueFontNoSize{#1}}

\newtheorem{prob}{Problem}
\newtheorem{fact}{Fact}

\newtheorem{assumption}{Assumption}

\makeatletter
\def\cl@chapter{\@elt {theorem}}
\makeatother
\usepackage{cleveref}
\crefname{equation}{}{}
\Crefname{equation}{Equation}{Equations}
\creflabelformat{equation}{(#1)}
\crefname{figure}{Fig.}{Figs.}
\Crefname{figure}{Figure}{Figures}
\crefname{section}{Sec.}{Secs.}
\Crefname{section}{Section}{Sections}
\crefname{subsection}{Sec.}{Secs.}
\Crefname{subsection}{Section}{Sections}
\crefname{subsubsection}{Sec.}{Secs.}
\Crefname{subsubsection}{Section}{Sections}
\crefname{table}{Table}{Tables}
\Crefname{table}{Table}{Tables}
\crefname{prob}{Problem}{Problems}
\Crefname{prob}{Problem}{Problems}
\crefname{algorithm}{Algorithm}{Algorithms}
\Crefname{algorithm}{Algorithm}{Algorithms}
\crefname{assumption}{Assumption}{Assumptions}
\Crefname{assumption}{Assumption}{Assumptions}
\crefname{lemma}{Lemma}{Lemmas}
\Crefname{lemma}{Lemma}{Lemmas}
\crefname{proposition}{Proposition}{Propositions}
\Crefname{proposition}{Proposition}{Propositions}

\definecolor{dark-gray}{gray}{0.5}
\usepackage[usenames,dvipsnames]{xcolor}

\def\distAnchor{\ell_{i}}

\def\primeTraveler{`prime traveler'}
\def\secondaryTraveler{`secondary traveler'}
\def\secondaryTravelers{`secondary travelers'}
\def\anchor{`anchor'}
\def\anchors{`anchors'}
\def\connector{`connector'}
\def\connectors{`connectors'}
\def\explorers{`explorers'}
\def\explorer{`explorer'}

\def\Explorer{`Explorer'}

\def\PickNplaceRobot{\Explorer{}~2} 
\def\pickNplaceRobot{\explorer{}~2} 
\def\CameraRobot{\Explorer{}~1} 
\def\cameraRobot{\explorer{}~1} 

\def\dampingForce{damping force}
\def\connectivityForce{generalized connectivity force}
\def\travelingForce{traveling force}
\def\connectivityAndDampingForce{connectivity and damping force}
\def\travelingAndConnectivityForce{traveling and connectivity force}

\newcounter{noindlistcounter}

\graphicspath{{./figures/}}

\begin{document}

\title{
Decentralized Simultaneous Multi-target Exploration using a Connected Network of Multiple Robots} 


\author{Thomas Nestmeyer \and Paolo Robuffo Giordano \and Heinrich H. B\"ulthoff \and Antonio Franchi}


\newcommand\linkedmail[1]{\href{mailto:#1}{\email{#1}}}

\institute{T.~Nestmeyer 
             - \linkedmail{tnestmeyer@tue.mpg.de}
             \at
             Max Planck Institute for Intelligent Systems,\\ Spemannstra{\ss}e 41, 72076 T\"ubingen, Germany
           \and
           P.~Robuffo~Giordano
             - \linkedmail{prg@irisa.fr}
             \at
             CNRS at Irisa and Inria Rennes Bretagne Atlantique,\\ Campus de Beaulieu, 35042 Rennes cedex, France
           \and
           H.~H.~B\"ulthoff 
             - \linkedmail{hhb@tuebingen.mpg.de}
             \at
             Max Planck Institute for Biological Cybernetics,\\ Spemannstra{\ss}e~38, 72076 T\"ubingen, Germany
          \and
           A.~Franchi 
             - \linkedmail{antonio.franchi@laas.fr}
             \at 
              CNRS, LAAS, 7 Av. du Colonel Roche, F-31400 Toulouse, France\\
              and Univ de Toulouse, LAAS, F-31400 Toulouse, France
\and
Simulations and experiments were performed while the authors were at the MPI for Biological Cybernetics, T\"ubingen, Germany.
}


\maketitle

\begin{abstract}
This paper presents a novel decentralized control strategy for a multi-robot system that enables parallel multi-target exploration while ensuring a time-varying connected topology in cluttered 3D environments. Flexible continuous connectivity is guaranteed by building upon a recent connectivity maintenance method, in which limited range, line-of-sight visibility, and collision avoidance are taken into account at the same time.
Completeness of the decentralized multi-target exploration algorithm is guaranteed by dynamically assigning the robots with different motion behaviors during the exploration task. One major group is subject to a suitable downscaling of the main \travelingForce{} based on the traveling efficiency of the current leader and the direction alignment between \travelingAndConnectivityForce{}. This supports the leader in always reaching its current target and, on a larger time horizon, that the whole team realizes the overall task in finite time.
Extensive Monte~Carlo simulations with a group of several quadrotor UAVs show the scalability and effectiveness of the proposed method and experiments validate its practicability. 
\end{abstract}

\section{Introduction}\label{sec:intro}

Success of multi-robot systems is based on their ability of parallelizing the execution of several small tasks composing a larger complex mission such as, for instance,
the inspection of a certain number of locations either generated off- or online during the robot motion (e.g., exploration, data collection, surveillance, large-scale medical supply or search and rescue~\citep{2006-HowParSuk,2009b-FraFreOriVen,2012b-PasFraBul,2008-MurTadNarJacFioChoErk,2014-FaiHol}).
In all these cases, a fundamental difference between a group of many single robots and a multi-robot system is the ability to communicate (either explicitly or implicitly) in order to then cooperate together towards a common objective.
Another distinctive characteristic in multi-robot systems is the
absence of central planning units, as well as all-to-all communication infrastructures, leading to a \emph{decentralized} approach for algorithmic design and implementations~\citep{1997-Lyn}.
While communication of a robot with every other robot in the group (via multiple hops) would still be possible as long as the group stays connected, in a decentralized approach each robot is only assumed able to communicate with the robots in its $1$-hop neighborhood (i.e., typically the ones spatially close by). This brings the advantage of scalability in communication and computation complexity when considering groups of many robots.

The possibility for every robot to share information (via, possibly, multiple hops/iterations) with any other robot in the group
is a basic requirement for typical multi-robot algorithms and, as well-known, it is directly related to the connectivity of the underlying \emph{graph} modeling inter-robot interactions.
Graph connectivity is a prerequisite to properly fuse the information collected by each robot, e.g., for mapping, localization, and for deciding the next actions to be taken. Additionally, many distributed algorithms like consensus~\citep{2004-OlfMur} and flooding~\citep{2001-LimKim}
require a connected graph for their successful convergence. Preserving graph connectivity during the robot motion is, thus, a fundamental requirement; however, connectivity maintenance may not be a trivial task in many situations, e.g., because of limited capabilities of onboard sensing/communication devices which can be hindered by constraints such as occlusions or maximum range.
Given the cardinal role of communication for the successful operation of a multi-robot team, it is then not surprising that a substantial effort has been spent over the last years for devising strategies able to preserve graph connectivity despite constraints in the inter-robot sensing/communication possibilities, see, e.g.,~\cite{2005-AntArrChiSet,2008-StuJadKum,2011-StuMichKumIsl,2012-PeiMut,2013e-RobFraSecBue}. In general, \emph{fixed topology} methods represent conservative strategies that achieve connectivity maintenance by restraining any pairwise link of the interaction graph to be broken during the task execution. 
A different possibility is to aim for \emph{periodical connectivity} strategies, where each robot can remain separated from the group during some period of time for then rejoining when necessary. 
\emph{Continuous connectivity} methods instead try to obtain maximum flexibility (links can be continuously broken and restored unlike in the fixed topology cases) while preserving at any time the fundamental ability for any two nodes in the group to share information via a (possibly multi-hop) path (unlike in periodical connectivity methods).

With respect to this state-of-the-art, the problem tackled by this paper is the design of a \emph{multi-target} exploration/visiting strategy for a team of mobile robots in a cluttered environment able to 
\begin{inparaenum}[\itshape i\upshape)]
\item allow visiting multiple targets at once (for increasing the efficiency of the exploration), while
\item always guaranteeing connectivity maintenance of the group despite some typical sensing/communication constraints representative of real-world situations,
\item without requiring presence of central nodes or processing units (thus, developing a fully \emph{decentralized} architecture), and
\item without requiring that all the targets are known at the beginning of the task (thus, considering \emph{online target generation}).
\end{inparaenum}

Designing a decentralized strategy that combines multi-target exploration and continuous connectivity maintenance is not trivial as these two goals impose often antithetical constraints. 
Several attempts have indeed been presented in the previous literature: a \emph{fixed-topology} and centralized method is presented in~\cite{2005-AntArrChiSet}, which, using a virtual chain of mobile antennas, is able to maintain the communication link between a ground station and a single mobile robot visiting a given sequence of target points. The method is further refined in~\citet{2006-AntArrChiSet}.
A similar problem is addressed in~\citet{2008-StuJadKum} by resorting to a partially centralized method where a linear programming problem is solved at every step of motion in order to mix the derivative of the second smallest eigenvalue of a weighted Laplacian (also known as algebraic connectivity, or Fiedler eigenvalue) and the k-connectivity of the system. A line-of-sight communication model is considered in~\cite{2011-StuMichKumIsl}, where a centralized approach, based on polygonal decomposition of the known environment, is used to address the problem of deploying a group of roving robots while achieving \emph{periodical connectivity}. 
The case of periodical connectivity is also considered in~\cite{2012b-PasFraBul} and~\cite{2012-HolSin}. The first paper optimally solves the problem of patrolling a set of points to be visited as often as possible. The second presents a heuristic algorithm exploiting the concept of implicit coordination. 
\emph{Continuous connectivity} between a group of robots exploring an unknown 2D environment and a single base station is considered in \cite{2010-PeiMutXi_}. The proposed exploration methodology, similar to the one presented in~\cite{2009b-FraFreOriVen}, is integrated with a centralized algorithm running on the base station and solving a variant of the Steiner Minimum Tree Problem. An extension of this approach to heterogeneous teams is presented in~\cite{2012-PeiMut}. \cite{2007-ZavPap} exploit a potential field approach to keep the second smallest eigenvalue of the Laplacian positive. The method is tested with ground robots in an empty environment and assumes that each robot has access to the whole formation for computing the connectivity eigenvalue and the associated potential field. It is therefore not scalable, because the strength of all links has to be broadcasted to all robots in the group. Continuous connectivity achieved by suitable mission planning is described in~\cite{2008-MosMonLag}, although this work does not allow for parallel exploration. Another method providing flexible connectivity based on a spring-damper system, but not able to handle significant obstacles, is reported in~\cite{2010-TarMosRiaVilMon}.

A decentralized strategy addressing the problem of continuous connectivity maintenance for a multi-robot team is considered in~\cite{2013e-RobFraSecBue}. 
In this latter work, the introduction of a sensor-based weighted Laplacian allows to distributively and analytically compute the anti-gradient of a generalized Fiedler eigenvalue. The connectivity maintenance action is further embedded with additional constraints and requirements such as inter-robot and obstacle collision avoidance, and a stability guarantee of the whole system, when perturbed by external control inputs for steering the whole formation, is also provided. Finally, apart for~\cite{2013e-RobFraSecBue}, all the previously mentioned continuous connectivity methods have only been applied to 2D-environment models.

In this work, we leverage upon the general decentralized strategy for connectivity maintenance of~\cite{2013e-RobFraSecBue} for proposing a
solution to the aforementioned problem of decentralized \emph{multi-target exploration} while coping with the (possibly opposing) constraints of continuous connectivity maintenance in a cluttered 3D environment.
The main contributions of this paper and features of the proposed algorithm can then be summarized as follows:
\begin{inparaenum}[\itshape i)]
\item decentralized and continuous maintenance of connectivity,
\item guarantee of collision avoidance with obstacles and among robots, 
\item possibility to take into account non-trivial sensing/communication models, including maximum range and line-of-sight visibility in 3D, 
\item stability of the overall multi-robot dynamical system, 
\item decentralized exploration capability,
\item possibility for more than one robot to visit different targets at the same time, 
\item online path planning without the need for any (centralized) pre-planning phase,
\item applicability to both 2D and 3D cluttered environments, 
and finally 
\item completeness of the multi-target exploration (i.e., all robots are guaranteed to reach all their targets in a finite time).
\end{inparaenum}
The items {\itshape i) - iv)} have already been tackled in~\cite{2013e-RobFraSecBue} and are here taken as a basis for our work. On the other hand, the combination of {\itshape i) - iv)} with the items {\itshape v) - ix)} is a novel contribution: to the best of our knowledge, our work is then the first attempt to propose a decentralized multi-target exploration algorithm possessing all the mentioned features altogether.

The rest of the paper is organized as follows: \cref{sec:prob} provides a formal description of the problem under consideration. The proposed algorithm is then thoroughly illustrated in \cref{sec:algo}. In \cref{sec:simExp}, we report the results of extensive Monte Carlo simulations and experiments with real quadrotors, and \cref{sec:conclusions} concludes the paper. In the Appendix, we recap the main features of the decentralized continuous connectivity method presented in~\cite{2013e-RobFraSecBue} which is extensively exploited in this paper. We finally note that a preliminary version of our work has been presented in~\cite{2013f-NesRobFra,2013l-NesRobFra}.

\section{System Model and Problem Setting}\label{sec:prob} 

We consider a group of $N$ robots operating in a 3D obstacle-populated environment and denote with $q_i\in\mathbb{R}^3$ the position of a
reference point of the $i$-th robot, $i=1,\ldots,N$, in an inertial world frame. We also let $\mathcal{O}$ be the set of obstacle points in the environment. Each robot $i$ is assumed to be endowed with an omnidirectional sensor able to measure the relative position $q_j-q_i$ of another robot $j$ provided that:
\begin{enumerate}
  \item $\|q_j-q_i\|<R_s$, where $R_s>0$ is the maximum sensing range of the sensor, and
  \item $\min_{\varsigma\in[0,1],o\in\mathcal{O}}\|q_i +
\varsigma (q_j - q_i) - o\| \ge R_o$, i.e., the line segment connecting $q_i$ to $q_j$ is at least at distance $R_o>0$ away from any obstacle
point.
\end{enumerate}
These two conditions account for two common characteristics of exteroceptive sensors, namely, presence of a limited sensing range $R_s$, and the need for a non-occluded line-of-sight visibility\footnote{More complex sensing models could also be taken into account, see~\cite{2013e-RobFraSecBue} for a discussion in this sense.}. We further assume that if the $i$-th robot can measure $q_j-q_i$ then it can also communicate with the $j$-th robot with negligible delays, that is, the sensing and communication graphs are taken coincident. This assumption is justified by the fact that communication typically relies on wireless technology, thus with a broader range than sensing and without the need for direct visibility to operate. The neighbors of the $i$-th robot are denoted with $\mathcal{N}_i(t)$, i.e., the (time-varying) set of robots whose relative position can be measured by the $i$-th robot at time $t$.

Each robot $i$ is also endowed with a sensor that measures the relative position $o-q_i$ of every obstacle point $o\in {\cal O}$ such that $\|o-q_i\|<R_m$, where $R_m>0$ is the maximum sensing range of this sensor.


Consider the time-varying (undirected) \emph{interaction graph} defined as $\mathcal{G}(t)=(\mathcal{V},\mathcal{E}(t))$, where $\mathcal{V}=\{1,\ldots,N\}$ and $\mathcal{E}(t)=\{(i,j)\,|\,j\in\mathcal{N}_i(t)\}$.
Preserving connectivity of $\mathcal{G}(t)$ for all $t$, allows every robot to communicate \emph{at any time} with any other robot in the network by means of a suitable multi-hop routing strategy, although due to efficiency and scalability reasons, it is always preferred to use one-hop communication when possible.

As previously stated, decentralized continuous connectivity maintenance is guaranteed by exploiting the method described in~\cite{2013e-RobFraSecBue}, which is based on a gradient-descent action that keeps positive the second smallest eigenvalue $\lambda_2$ of the \emph{sensor-based} weighted graph Laplacian~\citep{1973-Fie} (see Appendix~\ref{app:connectivity} for a formal definition).

Each $i$-th robot is finally endowed with a local motion controller able to let $q_i$ track any arbitrary desired $\bar{\mathcal{C}}^2$ trajectory $q_i(t)$ with a sufficiently small tracking error. This is again a well-justified assumption for almost all mobile robotic platforms of interest, and its validity will also be supported by the
experimental results of Sec.~\ref{sec:exp}. Following the control framework introduced in~\cite{2013e-RobFraSecBue}, the dynamics of $q_i$ is then modeled as the following second order system
\begin{equation}\label{eq:second_order}
\Sigma:
\begin{cases}
	M_i \dot v_i - f_i^B - f_i^\lambda = f_i \\
	\dot q_i = v_i 
\end{cases}\qquad
i=1,\ldots,N
\end{equation}
where $v_i\in \mathbb{R}^3$ is the robot velocity, $M_i\in\mathbb{R}^{3\times3}$ is its positive definite inertia matrix, and:
\begin{enumerate}
\item $f_i^B=-B_i v_i\in\mathbb{R}^3$ is the \emph{\dampingForce{}} (with $B_i\in\mathbb{R}^{3\times3}$ being a positive definite damping matrix) meant to represent both typical friction phenomena (e.g., wind/atmosphere drag in the case of aerial robots) and/or a stabilizing control action;
\item $f_i^\lambda\in\mathbb{R}^3$ is the \emph{\connectivityForce{}} whose decentralized computation and properties are thoroughly described in~\cite{2013e-RobFraSecBue} (a short recap is provided in Appendix~\ref{app:connectivity}); 
\item $f_i\in\mathbb{R}^3$ is the \emph{\travelingForce{}} used to actually steer the robot motion in order to execute the given task.
An appropriate design of $f_i$ is the main goal of this work. As will be clear in the following, special care must be taken in the design of $f_i$ to avoid, for instance, deadlocks situations in which the robot group `gets stuck'.
\end{enumerate}
The following fact, shown in~\cite{2013e-RobFraSecBue} and recalled in Appendix~\ref{app:connectivity}, holds:
\begin{fact} As long as $f_i$ keeps bounded, the action of the \connectivityForce{} $f_i^\lambda$ will always ensure obstacle and inter-robot collision avoidance \emph{and} continuous connectivity maintenance for the graph $\calG(t)$ despite the various sensing/communication constraints (in the worst case, by completely dominating the bounded $f_i$).
\end{fact}

To summarize, each robot has
\begin{inparaenum}[\itshape i)]
\item an accurate enough measurement of its own location,
\item an omnidirectional sensor which is able to measure relative positions of other robots and obstacles in its close proximity,
\item negligible (compared to the time scale of the robot motion) communication delays with all robots that it can sense/communicate with,
\item the ability to accurately track a smooth path with a force controller.
\end{inparaenum}

\subsection{Multi-target Exploration Problem}\label{sec:target_visiting} 

We consider the broad class of problems in which each robot runs a black-boxed algorithm that produces \emph{online}\footnote{By \emph{online} we mean that the targets are generated at runtime, thus precluding the presence of a preliminary phase in which the robots may \emph{plan in advance} the multi-target exploration action. Indeed, if all the targets are known beforehand, one could still apply our method but other planning strategies might potentially lead to better solutions.} a continually adjustable list of targets that have to be visited by the robot in the presented order. 
We refer to this algorithm as the \emph{target generator} of the $i$-th robot, and we also assume that the portion of the map needed to reach the next location from the current position $q_i$ is known to robot~$i$.
The target generator may represent a large variety of algorithms, such as pursuit-evasion~\citep{2012a-DurFraBul}, patrolling~\citep{2012b-PasFraBul}, exploration/mapping~\citep{2009b-FraFreOriVen,2005-BurMooStaSch}, mobile-ad-hoc-networking~\citep{2005-AntArrChiSet}, and active localization~\citep{2001-JenKri}. 
It might be a cooperative algorithm, or each robot could have a target generator with objectives that are independent from the other target generators. Another possibilty is to appoint a human supervisor as the target generator.

Depending on the particular application, the locations in the lists provided online by the target generators may, e.g., represent:
\begin{enumerate}
  \item view-points from where to perform the sensorial acquisitions,
  \item coordinates of objects that have to be picked up or dropped down,
  \item positions of some base stations located in the environment.
\end{enumerate}

We formally denote with $(z_i^1,\ldots,z_i^{m_i})\in\mathbb{R}^{3\times{m_i}}$ the list of $m_i$ locations provided by the $i$-th target generator. Additionally, we consider the possibility, for the target generator, to specify a time duration $\Delta t_i^k<\infty$ for which the $i$-th robot is required to stay close to the point $z_i^k$, with $k=1,\ldots,m_i$.
This quantity may represent, with reference to the previous examples, the time
\begin{enumerate}
  \item needed to perform a full sensorial acquisition,
  \item necessary to pick up/drop down an object,
  \item required to upload/download some information from a base station,
\end{enumerate}
and can also possibly be adjusted at runtime during the execution of the respective task.

Finally, we also introduce the concept of a \emph{cruise speed} $v_i^\text{cruise}>0$ that should be maintained by the $i$-th robot during the transfer phase from a point to the next one.

Given these modeling assumptions, the problem addressed in this paper can be formulated as follows:
\begin{prob}\label{prob:main}
Given a sequence of targets $z_i^1,\ldots,z_i^{m_i}$ (presented online) for every robot $i=1,\ldots,N$, together with the corresponding sequence of time durations $\Delta t_i^1,\ldots,\Delta t_i^{m_i}$ and a radius $R_z$,
design, for every $i=1,\ldots,N$, a decentralized feedback control law $f_i$ (i.e., a function using only information locally and $1$-hop available to the $i$-th robot) for system~\eqref{eq:second_order} which is bounded and such that, for the closed-loop trajectory $q_i(t,f_{i,{[0,t)}})$, there exists a time sequence $0<t_i^1<\ldots<t_i^{m_i}<\infty$ so that for all $k=1\ldots m_i$, robot $i$ remains for the duration $\Delta t_i^k$ within a ball of radius $R_z$ 
centered at~$z_i^k$, formally $\forall t\in[t_i^k,t_i^k+\Delta t_i^k]: \|q_i(t) - z_i^k\|< R_z$.
\end{prob}

\section{Decentralized Algorithm}\label{sec:algo}

In this section, we describe the proposed distributed algorithm aimed at generating a \travelingForce{} $f_i$ that solves~\cref{prob:main}. We note that the design of such an autonomous distributed algorithm requires special care: When added to the \connectivityForce{} in~\eqref{eq:second_order}, the \travelingForce{} $f_i$ should fully exploit the group capabilities to concurrently visit the targets of all robots whenever possible and, at the same time, should not lead to `\emph{local minima}', where the robots get stuck, due to the simultaneous presence of the hard connectivity constraint. While~\cite{2013e-RobFraSecBue} already gave an exact description of $f_i^B$ and $f_i^\lambda$, an application of $f_i$ was kept open. The main focus of this work is to define $f_i$ in such a way that the above mentioned challenges are properly addressed. 

In order to provide an overview of the several variables used in~\cref{sec:prob,sec:algo}, we included~\cref{tab:variableNames} for the reader's convenience.


\begin{table}
\caption{Meaning of the variable names.}
\label{tab:variableNames}
\centering
\begin{tabular}{cp{0.75\columnwidth}}
variable            & meaning                       \\ \hline
$N$                 & number of robots              \\
$q_i$               & position of $i$-th robot      \\
$v_i$               & velocity of $i$-th robot      \\
$\mathcal{O}$       & set of obstacle points        \\
$R_s$               & maximum sensing range         \\
$R_o$               & minimum distance to obstacle  \\
$R_c$               & minimum inter-robot distance  \\
$\mathcal{N}_i$     & neighbors of $i$-th robot     \\
$\mathcal{G}$       & interaction graph             \\
$\lambda_2$         & second smallest eigenvalue of the sensor-based weighted graph Laplacian    \\
$f_i^\lambda$       & \connectivityForce{}          \\
$f_i^B$             & \dampingForce{}               \\
$f_i$               & \travelingForce{}             \\
$z_i^k$             & $k$-th target of $i$-th robot \\
$\Delta t_i^k$      & amount of time to stay close to target $z_i^k$ \\
$R_z$               & maximum distance to target when anchored \\
$v_i^\text{cruise}$ & maximum cruise speed          \\
$\gamma_i$          & path to current target, starting from position of robot at time of computation \\
$q_i^\gamma$        & closest point of path from current position \\
$d_i^\gamma$        & length of remaining path \\
$R_\gamma$			& distance to path at which it should be re-planned \\
$\alpha_\Lambda$	& weighting of position vs. velocity error \\
$e_i$				& absolute tracking error of $i$-th robot along path \\ 
$(x_c,x_M)$         & tracking error bounds for the traveling efficiency \\
$\Lambda_i$         & traveling efficiency of $i$-th robot (i.e., tracking error nonlinearly scaled to $[0, 1]$ based on $x_c$, $x_M$) \\
$\hat\Lambda_p^i$   & estimation of the traveling efficiency of the \primeTraveler{} by the $i$-th robot \\
$\Theta_i$          & force direction alignment between connectivity and traveling force of $i$-th robot \\
$\sigma$            & weighting between the force direction alignment and the \primeTraveler{} traveling efficiency \\
$\rho_i$            & downscaling factor of a \secondaryTraveler{}, dependent on $\hat\Lambda_p^i$, $\Theta_i$ and $\sigma$
\end{tabular}
\end{table}


\subsection{Notation and Algorithm Overview}\label{sec:qualitative_desc}

As in any distributed design, several instances of the proposed algorithms run separately on each robot and locally exchange information with the `neighboring' instances via communication. 
Each instance is split into two concurrent routines: a \emph{planning algorithm} and a \emph{motion control algorithm} whose pseudocodes are given in~\cref{alg:logicLoop} and~\cref{alg:ctrlLoop}, respectively. 
The planning algorithm acts at a higher level and performs the following actions:
\begin{itemize}
\item it processes the targets provided by the target generator, 
\item it generates the desired path to the current target, and 
\item it selects an appropriate motion control behavior (see later).
\end{itemize}
The motion control algorithm acts at a lower level by specifying the \travelingForce{} $f_i$ as a function of the behavior and the planned path selected by the planning algorithm\footnote{The two routines can run at two different frequencies, typically slower for the planning loop and faster for the motion control loop.}.

The two algorithms have access to the same variables which are formally introduced as follows (see~\cref{fig:overview} for a graphical representation of some of these variables): the variable \varFontNoSizeSmall{targetQueue}$_i$ is filled online by the target generator and contains a list of future targets to be visited by the $i$-th robot. During the overall running time of the algorithms, the target generator of robot $i$ has access to the whole list \varFontNoSizeSmall{targetQueue}$_i$ (which can also be changed online if needed).
The current target for the $i$-th robot (i.e., the last target extracted from the first entry in \varFontNoSizeSmall{targetQueue}$_i$) is denoted with~$z_i$.
Variable $\gamma_i$ is a $\bar{\mathcal{C}}^2$ geometric path that leads from the current position $q_i$ of the $i$-th robot to the target $z_i$. In our implementation, we used B-splines~\citep{2008-BiaMel} in order to get a parameterized smooth path, but any other $\bar{\mathcal{C}}^2$ path would be appropriate. If the robot is not traveling towards any target, then $\gamma_i$ is set to $\valueFontNoSizeSmall{null}$.

\begin{figure}
\centering
\includegraphics[width=0.99\columnwidth]{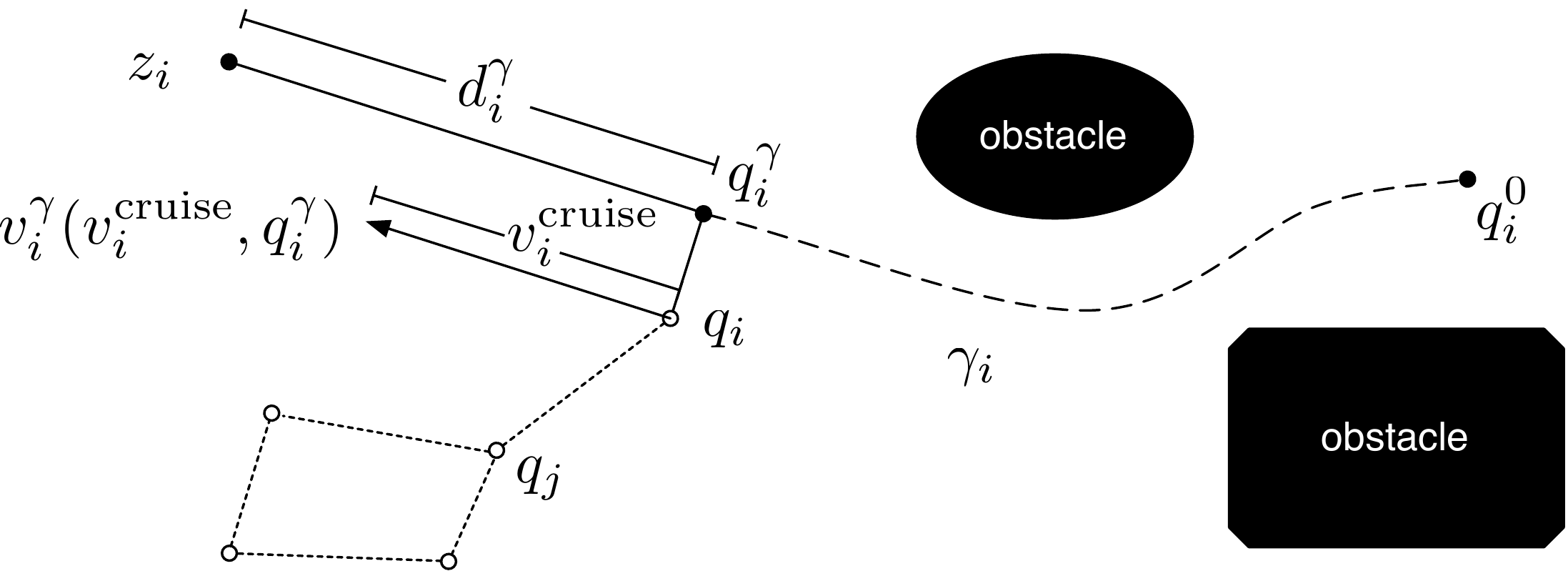}
\caption{Position $q_i$ and path $\gamma_i$ followed by a traveler from the point
$q_i^0$ to the current target $z_i$.
The solid part of the path represents the \emph{remaining path} which starts at the closest point on the path $q_i^\gamma$ and whose length is denoted
by $d_i^\gamma$.}
\label{fig:overview}
\end{figure}

With reference to~\cref{fig:overview}, we also denote with $q_i^\gamma$ the closest point of the path $\gamma_i$ to $q_i$, i.e., the solution of \linebreak $\arg\min_{p\in\gamma_i}\|p-q_i\|$. In case of multiple solutions, we choose the one with the largest arc-length, i.e., the one nearest to the target along the path. Therefore, the closest point $q_i^\gamma$ can be considered as unique in the following.
The portion of the path $\gamma_i$ from $q_i^\gamma$ to $z_i$ is referred to as the \emph{remaining path},
and its length is denoted with~$d^\gamma_i$.

The motion behavior of the $i$-th robot is determined by the variable \varFontNoSizeSmall{state}$_i$ that can take four possible values: 
\begin{itemize}
\item \valueFontNoSizeSmall{connector}
\item \valueFontNoSizeSmall{prime-traveler}
\item \valueFontNoSizeSmall{secondary-traveler}
\item \valueFontNoSizeSmall{anchor}.
\end{itemize}
The following provides a qualitative illustration of these motion behaviors, while a functional description is given in the next sections:
\begin{itemize}[$\bullet$]

\item \emph{\connector{}}: a robot in this state is not assigned any target by the target generator and therefore, its only goal is to help keeping the graph $\mathcal{G}$ connected. For this reason $f_i$ is set to zero and hence the robot is subject solely to the damping and \connectivityForce{} $f^\lambda_i$ in~(\ref{eq:second_order}); 

\item \emph{\primeTraveler{}}: a robot in this state travels towards its current target $z_i$ along the path $\gamma_i$ thanks to the force $f_i$. 
At the same time, the robot distributively broadcasts to every other robot a non-negative real number, denoted with $\Lambda_i$, that represents its \emph{traveling efficiency}, i.e., a measure how well it is able to follow its desired path while being influenced by the other robots in the group via the \connectivityForce{}~$f_i^\lambda$ (which is described in more detail later).
It is essential for the algorithm that only one \primeTraveler{} exists in the group at any time. Every other robot with an assigned target needs to be a \secondaryTraveler{} or \anchor{}. This feature will allow one robot (the \primeTraveler{}) to reach its target with a high priority, while the other robots will only be allowed to reach their own targets as long as this action does not hinder the \primeTraveler{} goal.

\item \emph{\secondaryTraveler{}}: a robot in this state travels towards its current target $z_i$ along the path $\gamma_i$ thanks to the force $f_i$. The robot keeps an internal estimation $\hat\Lambda_p^i$ of the traveling efficiency of the current \primeTraveler{}, and it scales down the intensity of its \travelingForce{} $f_i$ by an \emph{adaptive gain} $\rho_i$ whenever the action of $f_i$ is `too conflicting' w.r.t.~that of $f^\lambda_i$, or the \primeTraveler{} $\hat\Lambda_p^i$ drops lower than a given threshold. 

\item\emph{\anchor{}}: a robot in this state has reached the proximity of the target $z_i$. The force $f_i$ is then exploited in order to keep $q_i$ within a circle of radius $R_z$ centered at $z_i$ (i.e., the robot is `anchored' to the target), while waiting for the associated time $\Delta t_i$ to elapse. 
\end{itemize}

In order to obtain a better intuition of the roles of the robots, we suggest the reader to watch the ``Empty Space'' video available in the attached multimedia material\footnote{\url{http://homepages.laas.fr/afranchi/videos/multi_exp_conn.html}}.

To summarize this qualitative description, these behaviors are designed in such a way that the single \primeTraveler{} approaches its target with the highest priority, the \secondaryTravelers{} approach their targets as long as they have enough spatial freedom by the \connectivityForce{}, the \anchors{} stay close to the target until their task is completed, and the \connectors{} help the \secondaryTravelers{} in providing as much spatial freedom as possible while preserving the connectivity of the graph.

Whenever a robot moves, it may indirectly exert a certain \connectivityForce{} on all its neighbors because of the properties of $f_i^\lambda$ (i.e., for retaining generalized connectivity of the graph $\calG$ (see~\cite{2013e-RobFraSecBue} and Appendix~\ref{app:connectivity}). This connectivity action can possibly conflict with the \travelingForce{} $f_i$, and also prevent, in the worst case, fulfilment of the multi-target exploration task (e.g., the group falls in a local minimum because two robots start traveling in opposite directions over too large distances, thus threatening connectivity maintenance).

Since the \connectors{} implement $f_i=0$ by definition, they cannot directly hinder the \primeTraveler{} motion. In other words, a group made by all
\connectors{} and one \primeTraveler{} would always allow the \primeTraveler{} to reach its target.
Presence of \anchors{} can instead block the \primeTraveler{} because of the anchoring force 
which prevents them to move away from their targets. 
Nevertheless the anchoring phase can only last for a finite time $\Delta t_i^k$ after which the `anchor' changes state and is again free to move. 

No such mechanism is instead present for the \secondaryTravelers{} which would constantly attempt to move along their paths with a $\rho_i$ set to~$1$. As explained, if many robots are simultaneously traveling in arbitrary directions inside a cluttered environment, while also maintaining connectivity of $\calG$, the overall group motion can potentially (and quite easily) fall into a local minimum. 
The idea behind the gain $\rho_i$ is to then adaptively scale down the \travelingForce{} $f_i$ of the \secondaryTravelers{} whenever either 
\begin{inparaenum}[\itshape (i)\upshape]
\item the direction $f_i$ deviates too much from the connectivity force $f^\lambda_i$, or
\item the \primeTraveler{} motion is nevertheless too obstructed by the actions of the other \secondaryTravelers{} in the group. 
\end{inparaenum}
Consequently, this gain $\rho_i\in[0,1]$ is chosen so that the current \primeTraveler{} can always reach its target, no matter the motion planned by the \secondaryTravelers{} in the group. A formal description of this concept will be given in~\cref{sec:rho}.

\subsection{Start-up phase}\label{sec:startup}
\begin{procedure}[t]
\caption{Start-up for Robot $i$()}
\label{alg:init}

\footnotesize{


\uIf{\varFont{targetQueue}$_i$ is empty}
{
  $\gamma_i$ $\gets$ \valueFont{null} \;
  \varFont{state}$_i$ $\gets$ \valueFont{connector} \label{line:startup:connector}
}
\Else
{
  Extract first target from \varFont{targetQueue}$_i$ and save it as $z_i$ \;
  $\gamma_i$ $\gets$ Shortest obstacle-free path from $q_i$ to $z_i$\nllabel{line:startup:gamma_computed}
  
    Enroll in the list of Candidates to take part in the first distributed \primeTraveler{} election \label{line:startup:election}

    \uIf{  $i=\arg\min_{j\in\text{Candidates}}d^\gamma_j$ }
    {
      \varFont{state}$_i$ $\gets$ \valueFont{prime-traveler}
  }
  \Else
  {
    \varFont{state}$_i$ $\gets$ \valueFont{secondary-traveler}
  }
} 
$\hat\Lambda_p^i$ $\gets$ $0$ \nllabel{algo:init:estimate}\;
Run \cref{alg:logicLoop} and \cref{alg:ctrlLoop} in parallel\nllabel{algo:init:run_algos}\;
}
\end{procedure}

The Procedure~`\ref{alg:init}' performs the distributed initialization of the planning and motion control algorithms. Its pseudocode is quickly commented in the following.

At the beginning, if \varFontNoSizeSmall{targetQueue}$_i$ is empty then the path $\gamma_i$ is set to \valueFontNoSizeSmall{null} and \varFontNoSizeSmall{state}$_i$ to \valueFontNoSizeSmall{connector} (line~\ref{line:startup:connector}).
Otherwise the first target from \varFontNoSizeSmall{targetQueue}$_i$ is \linebreak extracted and saved in $z_i$. Then, the robot $i$ computes a $\bar{\mathcal{C}}^2$ shortest and obstacle-free path $\gamma_i$ that connects its current position $q_i$ with $z_i$ (line~\ref{line:startup:gamma_computed}). 
This path is generated with a two-step optimization method: 
first, the known portion of the map is discretized into an equally spaced grid in 3D with a cell size of $R_\text{grid}$. A cell is marked as occupied whenever an obstacle lies inside a radius of $R_\text{grid}$ around the cell. On this grid, a shortest path is found via~$A^*$. 
Then, the waypoints obtained from $A^*$ are approximated with a B-spline~\citep{2008-BiaMel} in order to remove corners from the path.
We note that, depending on the smoothing parameter, this approximation is not guaranteed to leave enough clearance from surrounding obstacles. Obstacle avoidance is nevertheless ensured thanks to 
the presence of the \connectivityForce{} that prevents any possible collisions by (possibly) locally adjusting the planned path when needed. 
As an alternative, one could also rely on the method proposed in~\cite{2012k-MasFraBueRob} for directly generating a smooth path with enough clearance from obstacles. 

Subsequently, the robot takes part in the distributed election of the first \primeTraveler{} (see~\cref{sec:election}). Depending on the outcome of this election, \varFontNoSizeSmall{state}$_i$ is set either to \valueFontNoSizeSmall{prime-traveler} or \valueFontNoSizeSmall{secondary-traveler}. 
 
At the end of the initialization procedure, the estimate $\hat\Lambda_p^i$ of the traveling efficiency of the current \primeTraveler{} is initialized to zero (line~\ref{algo:init:estimate}) for all robots, and the planning and motion control algorithms are both started (line~\ref{algo:init:run_algos}).

\subsection{Election of the \primeTraveler{}}\label{sec:election}

In a general election of a new \primeTraveler{}, the current \primeTraveler{} triggers the election process (line~\ref{line:logic:prime_start_election} of \cref{alg:logicLoop}), to which every \secondaryTraveler{} replies with its index and remaining path length, in order to be taken into the list of candidates (line~\ref{line:logic:secondary_election}). Since this election is a low-frequency event, we chose to implement it via a simple flooding algorithm~\citep{2001-LimKim}. Although this solution complies with the requirement of being decentralized, one could also resort to`smarter' distributed techniques such as~\citep{1997-Lyn}. The \primeTraveler{} then waits for $2(N-1)$ steps to collect these replies, being $2(N-1)$ the maximum number of steps needed to reach every robot with flooding and obtain a reply. The winner of this election is then the robot with the shortest remaining path length $d_i^\gamma$, i.e., the robot solving 
$\arg\min_{j\in\text{Candidates}}d^\gamma_j$. In the unlikely event of two (or more) robots having exactly the same remaining path length, the one with the lower index is elected. During the whole election process, the \primeTraveler{} keeps its role and only upon decision it abdicates by switching into the \anchor{} state. After announcing the winner, no \primeTraveler{} exists in the short time interval (at most $N-1$ steps) until the announcement reaches the winning \secondaryTraveler{}. This winning robot then switches into the \primeTraveler{} behavior. This mechanism makes sure that at most one \primeTraveler{} exists at any given time.

The \emph{first} election in the Start-up phase (see~\cref{sec:startup} and line~\ref{line:startup:election} in Procedure~`\ref{alg:init}') is handled slightly differently. Instead of the current \primeTraveler{} organizing the election, robot $1$ is always assigned the role of host and, instead of the only \secondaryTravelers{} replying, every robot with an assigned target replies with its index and remaining path length (including robot $1$ if it has an assigned target).

\subsection{Planning Algorithm}\label{sec:logicLoop}
 
\begin{algorithm}[t]
\caption{Planning for Robot $i$}
\label{alg:logicLoop}

\footnotesize{

\While{\valueFont{true}\nllabel{algo:logic:big_while}}
{
	\Switch{\varFont{state}$_i$}
	{
		\Case{\valueFont{connector}}
		{
			\If {\varFont{targetQueue}$_i$ is not empty						\nllabel{line:logic:queue_not_empty}}
			{
	  			Extract the next target from \varFont{targetQueue}$_i$ and save it as $z_i$	\nllabel{line:logic:extract_target_and_save}
	   
	  			$\gamma_i$ $\gets$ Shortest obstacle-free path from $q_i$ to $z_i$	\nllabel{line:logic:gamma_computed}
	   
		 		\uIf{	There is no \primeTraveler{} in the group}
		 		{
					\varFont{state}$_i$ $\gets$\valueFont{prime-traveler}	\nllabel{line:logic:from_connector_to_prime_traveler} 
	    		}
	    		\Else
	    		{
					\varFont{state}$_i$ $\gets$\valueFont{secondary-traveler}	\nllabel{line:logic:from_connector_to_secondary_traveler} 
	    		}
	
	  		}
		}
	  
		\Case{\valueFont{prime-traveler}\nllabel{line:logic:case_prime}}
		{
			\If {$\|q_i-z_i\| < R_z$\nllabel{line:logic:prime_check_target}}
	    	{ 
	    		$\gamma_i$ $\gets$ \valueFont{null}							\label{line:logic:prime_reset_path}
	       
				Permit \primeTraveler{} candidacy within timeout				\label{line:logic:prime_start_election}
				
				\varFont{state}$_i$ $\gets$ \valueFont{anchor}				\label{line:logic:prime_to_anchor} 
			}
		}
	  
		\Case{\valueFont{secondary-traveler}									\label{line:logic:case_secondary}}
		{
			\If {$\|q_i-q_i^\gamma\|>R_\gamma$							\label{line:logic:case_secondary_check_track}}
			{
				$\gamma_i$ $\gets$ Shortest obstacle-free path from $q_i$ to $z_i$	\nllabel{line:logic:secondary_recomputes_gamma}\; 
			}
	
	        \uIf {$\|q_i-z_i\| < R_z$\nllabel{line:logic:secondary_check_target}}
	    	{
	    		$\gamma_i$ $\gets$ \valueFont{null}							\label{line:logic:secondary_reset_path} 
	    		
	    		\varFont{state}$_i$ $\gets$ \valueFont{anchor}				\label{line:logic:secondary_to_anchor}
	    	}
	   		\ElseIf {\primeTraveler{} candidacy is allowed					\label{line:logic:secondary_checks_election}}
	      	{ 
		   		Enroll in the list of Candidates to take part in the distributed \primeTraveler{} election	\label{line:logic:secondary_election}
	      
				\If {$i=\arg\min_{j\in \text{\rm Candidates}}d^\gamma_j$		\label{line:logic:secondary_winner}}
				{
	  	        	\varFont{state}$_i$ $\gets$ \valueFont{prime-traveler}	\label{line:logic:secondary_to_prime}
				}
			}
		}
	    
	    \Case{\valueFont{anchor}												\label{line:logic:case_anchor}}
		{
			\If{task at target $z_i$ is completed								\label{line:logic:anchor_check}}
			{
				\varFont{state}$_i$ $\gets$ \valueFont{connector} 
			}
		}
	}
}

}
\end{algorithm}
 
In this section, we describe in detail the execution of \cref{alg:logicLoop} running on the $i$-th robot, whose logical flow is provided in \Cref{fig:stateMachine} as a graphical representation. The algorithm consists of a continuous loop where different decisions are taken according to the value of \varFontNoSizeSmall{state}$_i$ and according to the following different behaviors:

\begin{figure}
\includegraphics[width=\columnwidth]{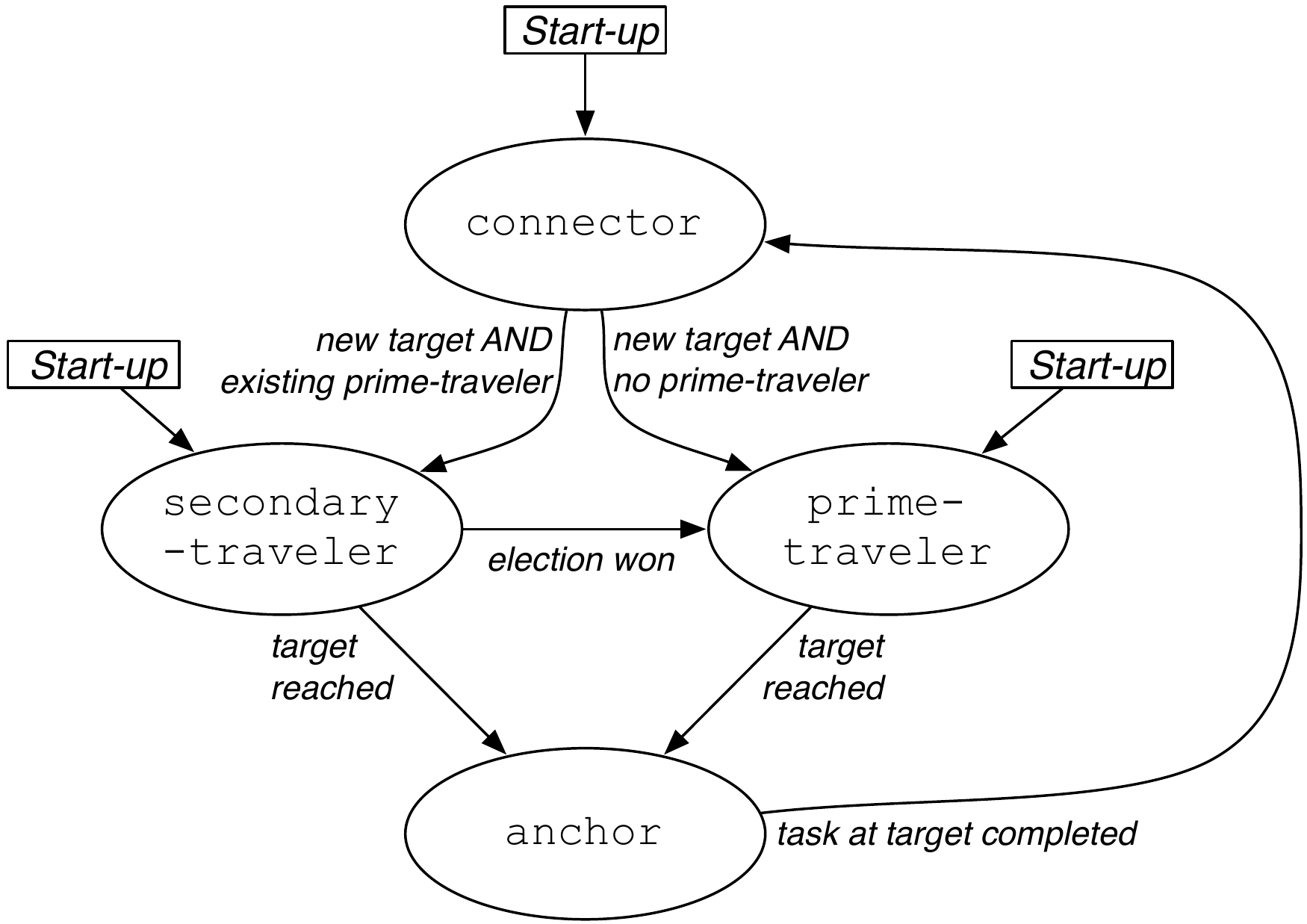}
\caption{State machine of the \cref{alg:logicLoop}.}
\label{fig:stateMachine}
\end{figure}

\subsubsection{case \valueFontNoSizeSmall{connector}}

If \varFontNoSizeSmall{state}$_i$ is set to \valueFontNoSizeSmall{connector} then \mbox{\varFontNoSizeSmall{targetQueue}$_i$} is checked. In case of an empty queue, \varFontNoSizeSmall{state}$_i$ remains \valueFontNoSizeSmall{connector}, otherwise the next target is extracted from the queue and saved in $z_i$ (line~\ref{line:logic:extract_target_and_save}). Then the $i$-th robot computes a $\bar{\mathcal{C}}^2$ shortest and obstacle-free path $\gamma_i$ connecting the current robot position $q_i$ with $z_i$ (line~\ref{line:logic:gamma_computed}) implementing what was previously described in the start-up procedure. 
Finally, the robot changes the value of \varFontNoSizeSmall{state}$_i$ in order to track $\gamma_i$. In particular, if no \primeTraveler{} is present in the group, then \varFontNoSizeSmall{state}$_i$ is set to \valueFontNoSizeSmall{prime-traveler} (line~\ref{line:logic:from_connector_to_prime_traveler}). Otherwise, \varFontNoSizeSmall{state}$_i$ is set to \valueFontNoSizeSmall{secondary-traveler} (line~\ref{line:logic:from_connector_to_secondary_traveler})\footnote{Presence of a \primeTraveler{} can be easily assessed in a distributed way by, e.g., flooding~\citep{2001-LimKim} on a low frequency.}. 

\subsubsection{case \valueFontNoSizeSmall{prime-traveler}}

When \varFontNoSizeSmall{state}$_i$ is set to \valueFontNoSizeSmall{prime-traveler} (line~\ref{line:logic:case_prime}) 
and the current position $q_i$ is closer than $R_z$ to the target $z_i$ (line~\ref{line:logic:prime_check_target}), the following actions are performed:
\begin{itemize}
\item the path $\gamma_i$ is reset to \valueFontNoSizeSmall{null} (line~\ref{line:logic:prime_reset_path}),
\item a new distributed \primeTraveler{} election as described in~\cref{sec:election} is announced (line~\ref{line:logic:prime_start_election}),
\item the robot abdicates the role of \primeTraveler{} and \varFontNoSizeSmall{state}$_i$ is set to \valueFontNoSizeSmall{anchor} (line~\ref{line:logic:prime_to_anchor}).
\end{itemize}
If, otherwise, $z_i$ is still far from the current robot position $q_i$, then \varFontNoSizeSmall{state}$_i$ remains unchanged and the robot continues to travel towards its target.

\subsubsection{case \valueFontNoSizeSmall{secondary-traveler}} \label{sec:caseSecondary}

When \varFontNoSizeSmall{state}$_i$ is \valueFontNoSizeSmall{secondary-traveler} (line~\ref{line:logic:case_secondary}) the distance $\|q_i^\gamma-q_i\|$ to the (closest point on the) path is checked (line~\ref{line:logic:case_secondary_check_track}). If this distance is larger than the threshold $R_\gamma$, the robot replans a path from its current position $q_i$ (line~\ref{line:logic:secondary_recomputes_gamma}). This re-planning phase is necessary since a \secondaryTraveler{} could be arbitrarily far from the previously planned $\gamma_i$ because of the `dragging action' of the current \primeTraveler{}. \Cref{sec:ctrlLoop} will elaborate more on this point. 
Subsequently, if $q_i$ is closer than $R_z$ to the target $z_i$ (line~\ref{line:logic:secondary_check_target}), the path $\gamma_i$ is reset to \valueFontNoSizeSmall{null} (line~\ref{line:logic:secondary_reset_path}) and \varFontNoSizeSmall{state}$_i$ is set to \valueFontNoSizeSmall{anchor} (line~\ref{line:logic:secondary_to_anchor}). Otherwise, if the target is still far away, the robot checks whether the \primeTraveler{} abdicated and announced an election of a new \primeTraveler{} (line~\ref{line:logic:secondary_checks_election}). If this was the case, the robot takes part in the election (line~\ref{line:logic:secondary_election}) as described in~\cref{sec:election}. If the robot wins the election (line~\ref{line:logic:secondary_winner}), \varFontNoSizeSmall{state}$_i$ is set to \valueFontNoSizeSmall{prime-traveler} (line~\ref{line:logic:secondary_to_prime}) otherwise it remains set to \valueFontNoSizeSmall{secondary-traveler}. 

\subsubsection{case \valueFontNoSizeSmall{anchor}}

The last case of \cref{alg:logicLoop} is when \varFontNoSizeSmall{state}$_i$ is \valueFontNoSizeSmall{anchor} (line~\ref{line:logic:case_anchor}). The robot remains in this state until the task at target $z_i$ is completed (line~\ref{line:logic:anchor_check}), after which \varFontNoSizeSmall{state}$_i$ is set to \valueFontNoSizeSmall{connector}.

\subsection{Completeness of the Planning Algorithm}

Before illustrating the \emph{motion control algorithm}, we state some important properties that hold during the whole execution of the planning algorithm.
\begin{proposition}\label{lemma:onlyOnePrimeTraveler}
If there exists at least one target in one of the \varFontNoSizeSmall{targetQueue}$_i$, then exactly one \primeTraveler{} will be elected at the beginning of the operation. 
Furthermore, this \primeTraveler{} will keep its state until being closer than $R_z$ to its assigned target. In the meantime no other robot can become \primeTraveler{}.
\end{proposition}
\begin{proof}
The start-up procedure guarantees that, if there exists at least one target in at least one of the \varFontNoSizeSmall{targetQueue}$_i$, the group of robots includes exactly one \primeTraveler{} and no \anchor{} at the beginning of the task. Any other robot is either a \connector{} or \secondaryTraveler{} depending on the corresponding availability of targets. During the execution of~\cref{alg:logicLoop}, a robot can only switch into \primeTraveler{} when being a \connector{} or a \secondaryTraveler{}. As long as there exists a \primeTraveler{} in the group, a \connector{} cannot become a \primeTraveler{}. Furthermore, a \secondaryTraveler{} becomes a \primeTraveler{} only if it wins the election announced by the \primeTraveler{}. Since the \primeTraveler{} allows for this election only when in the vicinity of its target (within the radius $R_z$), the claim directly follows.
\end{proof}

Using this result, the following proposition shows that \cref{alg:logicLoop} is actually guaranteed to complete the multi-target exploration in the following sense: when presented with a finite amount of targets, all targets of all robots are guaranteed to be visited in a finite amount of time. In order to show this result, an assumption on the robot motion controller is needed.

\begin{assumption}\label{assumption:primeTravelerReachesTarget}
In a group of robots with exactly one \primeTraveler{}, the adopted motion controller is such that the \primeTraveler{} is able to arrive closer than $R_z$ to its target in a finite amount of time regardless of the location of the targets assigned to the other robots.
\end{assumption}

In~\cref{sec:rho} we discuss in detail how the motion controller introduced in the next section~\ref{sec:ctrlLoop} meets Assumption~\ref{assumption:primeTravelerReachesTarget}.

\begin{proposition}
Given a finite number of targets and a motion controller fulfilling~\cref{assumption:primeTravelerReachesTarget}, the whole multi-target exploration task is completed in a finite amount of time as long as the local tasks at every target can be completed in finite time.
\end{proposition}
\begin{proof}
In the trivial case of no targets, the multi-target exploration task is immediately completed. Let us then assume presence of at least one target. \cref{lemma:onlyOnePrimeTraveler} guarantees existence of exactly one \primeTraveler{} at the beginning of the planning algorithm, and that such a \primeTraveler{} will keep its role until reaching its target, an event that, by virtue of Assumption~\ref{assumption:primeTravelerReachesTarget}, happens in finite time.
At this point, assuming as a \emph{worst case} that no \secondaryTraveler{} has reached and cleared its own target in the meantime,
one of the following situations may arise:
\begin{enumerate}
  \item There is at least one \secondaryTraveler{}. The \secondaryTraveler{} closest to its target becomes the new \primeTraveler{} in the triggered election, and it then starts traveling towards its newly assigned target until reaching it in finite time (\cref{lemma:onlyOnePrimeTraveler} and Assumption~\ref{assumption:primeTravelerReachesTarget})
  \item There is no \secondaryTraveler{} and no other \anchor{} besides the former \primeTraveler{}. In this case, no other robot has targets in its queue as, otherwise, at least one \secondaryTraveler{} would exist. Therefore, after completing its task at the target location (a finite duration), the former \primeTraveler{} and now \anchor{} becomes \connector{} and, in case of additional targets present for this robot, it switches back into being a \primeTraveler{} and travels towards the new targets in a finite amount of time as in case~1.
  \item There is no \secondaryTraveler{}, but at least one other \anchor{}. This situation can be split again into two sub-cases: 
  \begin{enumerate}
  \item there exists at least one \anchor{} with a future target in its queue. Then, after a finite time, this \anchor{} becomes \secondaryTraveler{} and case~1 holds;
  \item there is no \anchor{} with a future target in its queue. Then, after a finite time, all \anchors{} have completed their local tasks and case~2 holds.
  \end{enumerate}
\end{enumerate}
In all cases, therefore, one target is visited in finite time by the current \primeTraveler{}.
Repeating this loop finitely many times, for all the (finite number of) targets, allows to conclude that \emph{all} targets will be visited in a finite amount of time, thus showing the completeness of the planning algorithm.

If a \secondaryTraveler{} already reaches its target while the \primeTraveler{} is active, the aforementioned worst case assumption is not valid anymore. But since, in this case, the target of the \secondaryTraveler{} is already cleared, the total number of iterations is even smaller than in the previous worst case, thus still resulting in a finite completion time.
\end{proof}

\subsection{Motion Control Algorithm}\label{sec:ctrlLoop}

\begin{algorithm}[t]
\caption{Motion Control for Robot $i$}
\label{alg:ctrlLoop}
\footnotesize{

\While{\valueFont{true}}
{
  \Switch{\varFont{state}$_i$}
  {
    \Case{\valueFont{connector}}
    {
      Update $\hat\Lambda_p^i$ using $\dot{\hat\Lambda}_p^i=k_\Lambda\sum_{j\in\mathcal{N}_i}(\hat\Lambda_p^j - \hat\Lambda_p^i)$ \label{algo:ctrl:connector_consensus}
        
      $f_i \gets 0$ \label{algo:ctrl:connector_force}
    }
    \Case{\valueFont{prime-traveler} }
    { 
      $\hat\Lambda_p^i \gets \Lambda_i$ using \eqref{eq:velmatch} \label{line:ctrl:prime_estimate}

   
      $f_i \gets f_\text{travel}(q_i,\gamma_i,v_i^\text{cruise})$, using~\eqref{eq:ctrl:track} \nllabel{line:ctrl:prime_tracking}
    }
    \Case{\valueFont{secondary-traveler}}
    { 
      Update $\hat\Lambda_p^i$ using $\dot{\hat\Lambda}_p^i=k_\Lambda\sum_{j\in\mathcal{N}_i}(\hat\Lambda_p^j - \hat\Lambda_p^i)$ \label{line:ctrl:secondary_estimate}
      
      $f_i \gets \rho_i f_\text{travel}(q_i,\gamma_i,v_i^\text{cruise})$, using~\eqref{eq:ctrl:track} and~\eqref{eq:gain} \nllabel{line:ctrl:secondary_tracking}
    }
    \Case{\valueFont{anchor}}
    { 
      Update $\hat\Lambda_p^i$ using $\dot{\hat\Lambda}_p^i=k_\Lambda\sum_{j\in\mathcal{N}_i}(\hat\Lambda_p^j - \hat\Lambda_p^i)$ \label{line:ctrl:anchor_estimate}
 
      $f_i \gets f_\text{anchor}(q_i,z_i,R_z) $, as per~\eqref{eq:anchor} \nllabel{algo:ctrl:hooked_leader} \label{line:ctrl:anchor_tracking} 
    }
  }
}

}
\end{algorithm}

With reference to~\cref{alg:ctrlLoop}, we now describe the motion control algorithm that runs in parallel to the planning algorithm on the $i$-th robot, and whose goal is to determine a \travelingForce{} $f_i$ that can meet Assumption~\ref{assumption:primeTravelerReachesTarget}. 
The algorithm consists of a continuous loop, as before, in which the force $f_i$ is computed according to the behavior encoded in the variable \varFontNoSizeSmall{state}$_i$ determined by~\cref{alg:logicLoop}:

\subsubsection{case \valueFontNoSizeSmall{connector}}

If \varFontNoSizeSmall{state}$_i$ is set to \valueFontNoSizeSmall{connector}, the estimate $\hat\Lambda_p^i$ of the traveling efficiency of the current \primeTraveler{} is updated with a consensus-like algorithm (line~\ref{algo:ctrl:connector_consensus}) that will be described in the next \cref{sec:rho}. The \travelingForce{} $f_i$ is in this case simply set to~$0$~(line~\ref{algo:ctrl:connector_force}).
It is worth mentioning that $f_i=0$ does not mean the $i$-th robot will not move, since a \connector{} is still dragged by the other travelers via the \connectivityForce{} (according to~\eqref{eq:second_order}, the \connectors{} are still subject to $f_i^\lambda$ and $f_i^B$).

\subsubsection{case \valueFontNoSizeSmall{prime-traveler}}

If \varFontNoSizeSmall{state}$_i$ is set to \valueFontNoSizeSmall{prime-traveler}, the estimate $\hat\Lambda_p^i$ is set to the true traveling efficiency $\Lambda_i$, defined by~\eqref{eq:velmatch} (line~\ref{line:ctrl:prime_estimate}). Afterwards (line~\ref{line:ctrl:prime_tracking}) the robot sets
\begin{equation}\label{eq:travelerForcePrime}
f_i = f_\text{travel}(q_i,\gamma_i,v_i^\text{cruise}) \, ,
\end{equation}
where $f_\text{travel}(q_i,\gamma_i,v_i^\text{cruise})\in\mathbb{R}^3$
is a proportional, derivative and feedforward controller meant to travel along $\gamma_i$ at a given cruise speed $v_i^\text{cruise}$:
\begin{align}
\label{eq:ctrl:track}
f_\text{travel}(q_i,\gamma_i,v_i^\text{cruise}) = {}&a_i^\gamma(v_i^\text{cruise},q_i^\gamma) \notag\\&+ k_v (v_i^\gamma(v_i^\text{cruise},q_i^\gamma) - \dot q_i) \\&+ k_p(q_i^\gamma - q_i).\notag
\end{align}
Here, $k_p$ and $k_v$ are positive gains, $q_i^\gamma$ is the point on $\gamma_i$ closest to $q_i$ (see Fig.~\ref{fig:overview}), $v_i^\gamma(v_i^\text{cruise},q_i^\gamma)$ is the velocity vector of a virtual point traveling along $\gamma_i$ and passing at $q_i^\gamma$ with tangential speed $v_i^\text{cruise}$, and $a_i^\gamma(v_i^\text{cruise},q_i^\gamma)$ is the acceleration vector of the same point. It is straightforward to analytically compute both the velocity and the acceleration from $v_i^\text{cruise}$, given the spline representation of the curve~\citep{2008-BiaMel}.

\subsubsection{case \valueFontNoSizeSmall{secondary-traveler}}

If \varFontNoSizeSmall{state}$_i$ is set to \valueFontNoSizeSmall{secondary-traveler}, the estimate $\hat\Lambda_p^i$ is updated with a consensus-like protocol (line~\ref{line:ctrl:secondary_estimate}). Then (line~\ref{line:ctrl:secondary_tracking}) the robot sets
\begin{equation}\label{eq:travelerForceWithRho}
f_i = \rho_i f_\text{travel}(q_i,\gamma_i,v_i^\text{cruise}),
\end{equation}
where $f_\text{travel}$ is defined as in~\eqref{eq:ctrl:track} and $\rho_i\in[0,1]$ is an adaptive gain
meant to scale down the intensity of the action of $f_\text{travel}(q_i,\gamma_i,v_i^\text{cruise})$ whenever
\begin{inparaenum}[\itshape (i)\upshape] 
\item its alignment is too conflicting with the \connectivityForce{} $f^\lambda_i$ or
\item the \primeTraveler{} is not able to efficiently travel along its path because its reached speed is too low compared to its desired cruise speed.
\end{inparaenum}
\Cref{sec:rho} is dedicated to provide details on choosing an effective $\rho_i$.

\subsubsection{case \valueFontNoSizeSmall{anchor}}

If \varFontNoSizeSmall{state}$_i$ is set to \valueFontNoSizeSmall{anchor}, the estimate $\hat\Lambda_p^i$ is again updated using a consensus-like protocol (line~\ref{line:ctrl:anchor_estimate}). Then (line~\ref{algo:ctrl:hooked_leader}) the force $f_i$ is set as
\begin{equation}\label{eq:anchor}
f_i = f_\text{anchor}(q_i,z_i,R_z) = -\parder{V_\text{anchor}^{R_z}(\|q_i - z_i\|)}{q_i}
\end{equation}
where $V_\text{anchor}^{R_z}:[0,R_z)\to [0,\infty)$ is a monotonically increasing potential function of the distance $\distAnchor = \|q_i - z_i\|$ between the robot position $q_i$ and the target $z_i$, and such that $V_\text{anchor}^{R_z}(0)=0$ and $\lim_{\distAnchor\nearrow R_z} V_\text{anchor}^{R_z}(\distAnchor)= \infty$.
Under the action of $f_\text{anchor}(q_i,z_i,R_z) $ the position $q_i$ is then guaranteed to remain confined within
a sphere of radius $R_z$ centered at $z_i$ until the local task at the target location is completed.
In our simulations and experiments we employed
\begin{equation}\label{eq:potential_Vanchor}
V_\text{anchor}^{R_z}(\distAnchor) = -k_z \frac{2R_z}{\pi} \ln\left(\cos\left(\frac{\distAnchor \pi}{2R_z}\right)\right)
\end{equation}
where $k_z$ is an arbitrary positive constant.
The shape of this function is shown in~\cref{fig:controlStrategy:potential_Vanchor} and the associated $f_\text{anchor}$ is
\begin{equation}\label{eq:f_anchor}
f_\text{anchor}(q_i,z_i,R_z) = -k_z \tan\left(\frac{\distAnchor \pi}{2R_z}\right)\frac{q_i - z_i}{\distAnchor}\, .
\end{equation}

\begin{figure}
\centering
\includegraphics[width=0.8\columnwidth]{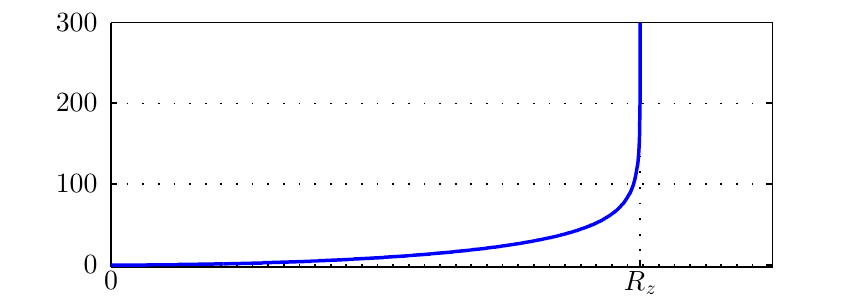}
\caption{Shape of the function $V_\text{anchor}^{R_z}(\distAnchor)$ defined in~\eqref{eq:potential_Vanchor} that is~$0$ on the target~$z_i$ itself ($\distAnchor=0$) and grows unbounded at the border of a sphere with radius~$R_z$.}
\label{fig:controlStrategy:potential_Vanchor}
\end{figure}

\subsection{Traveling Efficiency, Force Alignment and Adaptive Gain}\label{sec:rho}
We now describe how the estimation of the \emph{traveling efficiency} $\Lambda_i$ of all robots and the \emph{adaptive gain} $\rho_i$ of a \secondaryTraveler{}, used in \cref{alg:ctrlLoop}, are actually computed. Remember that the idea behind the gain $\rho_i$ is to adaptively scale down the \travelingForce{} $f_i$ of a \secondaryTraveler{} whenever 
\begin{inparaenum}[\itshape (i)\upshape]
\item the alignment of $f_i$ and the \connectivityForce{} $f^\lambda_i$ is too different, or
\item the traveling efficiency of the \primeTraveler{} is too low.
\end{inparaenum}

Therefore the design of $\rho_i$ aims at guaranteeing that the current \primeTraveler{} can always reach its target, whatever the motion planned by the other robots in the group are, and thus in fact enforces Assumption~\ref{assumption:primeTravelerReachesTarget}.

We recall that we provide a compendium of all important variables in~\cref{tab:variableNames}.

In order to implement the desired behavior we introduce two functions: 
\begin{align}
\Theta &:\mathbb{R}^3\times\mathbb{R}^3\to[0,1]\notag\\
\Lambda &: \mathbb{R}^+_0\times K^*\to[0,1]\notag
\end{align} 
where $K^* = \{(x_c,x_M)\in\mathbb{R}^2\,|\,0\le x_c<x_M\}$, defined as:
\begin{align}
 \Theta(x,y) &= 
   \begin{cases}
   	\frac{1}{2}\left(1+\frac{x^T y}{\|x\| \|y\|} \right) & x\neq0, y\neq 0\\
	1 & \text{otherwise}
\end{cases}\label{eq:dirmatch}\\
  \Lambda(x,x_c,x_M) &= 
  \begin{cases}
1						& x\in[0,x_c]\\
\frac{1}{2}+\frac{\cos\left(\frac{x-x_c}{x_M - x_c}\pi\right)}{2}		& x\in(x_c,x_M)\\
0						& x\in[x_M,\infty).
\end{cases}\label{eq:velmatch}
\end{align}
Function $\Theta(x,y)$ represents a `measure' of the direction alignment of the two non-zero 3D vectors $x$ and~$y$. In particular, $\Theta(x,y)$ is $1$ if $x$ and $y$ are parallel with the same direction, $\tfrac{1}{2}$~if they are orthogonal, and $0$ if they are parallel with opposite direction. Note that $\Theta(x,y)$ is equivalent to $\tfrac{1}{2}(1+\cos\theta)$ with $\theta$ being the angle between vectors $x$ and $y$. 

Function $\Lambda(x,x_c,x_M)$ `measures' how small $x$ is. If $x\leq x_c$ then $x$ is considered `small enough' and, therefore, $\Lambda = 1$. If $x\in(x_c,x_M)$ then $\Lambda$ strictly monotonically varies from $1$~to~$0$. If $x\geq x_M$, then $\Lambda=0$. The shape of $\Lambda$ is depicted in~\cref{fig:controlStrategy:Lambda}.

\begin{figure}[t]
\centering
\includegraphics[width=0.8\columnwidth]{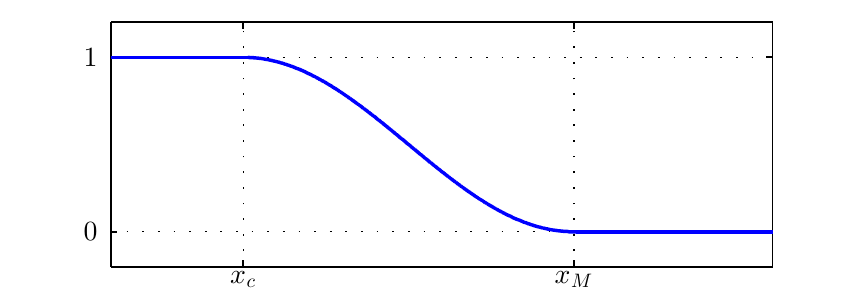}
\caption[Sketch of the function $\Lambda$]{Sketch of the function $\Lambda(x,x_c,x_M)$ for fixed $x_c$ and $x_M$.}
\label{fig:controlStrategy:Lambda}
\end{figure}

Having introduced these functions, we now define the \emph{force direction alignment} of the $i$-th robot as
\begin{equation}\label{eq:Theta}
\Theta_i=\Theta(f_i^\lambda,f_\text{travel}(q_i,\gamma_i,v_i^\text{cruise})),
\end{equation}
and note that $\Theta_i$ can be locally computed by the $i$-th robot. The quantity $\Theta_i$ thus represents an index in
$[0,1]$ measuring the degree of conflict among the directions of the \connectivityForce{} and the \travelingForce{}.

When $\gamma_i\neq\ $\valueFontNoSizeSmall{null}, we also define the absolute tracking error as 
\begin{align}
e_i = (1-\alpha_{\Lambda})\|v_i^\gamma(v_i^\text{cruise},q_i^\gamma) - v_i\| +\alpha_{\Lambda} \|q_i^\gamma - q_i\|,
\end{align}
with $\alpha_{\Lambda} \in[0,1]$ being a constant parameter modulating the importance of the velocity tracking error w.r.t. the position tracking error. The \emph{traveling efficiency} is then defined as
\begin{equation}\label{eq:traveling_efficiency}
\Lambda_i = \Lambda\left(e_i, x_c, x_M\right),
\end{equation}
where $0\le x_c < x_M < \infty$ are two user-defined thresholds representing the point at which the traveling efficiency $\Lambda_i$ starts to decrease and the maximum tolerated error after which the traveling efficiency vanishes. In this way it is possible to evaluate how well a traveler can follow its desired planned path according to a suitable combination of velocity and position accuracy. It is important to note that the value $\Lambda_i=1$ does not imply an exact tracking of the path, but it still allows a small tracking tolerance (dependent on the parameter $x_c$). Similarly, the value $\Lambda_i=0$ does not imply a complete loss of path tracking, but it represents the possibility of a tracking error higher than a maximum threshold (dependent on $x_M$).

In order to meet Assumption~\ref{assumption:primeTravelerReachesTarget}, we are only interested in the traveling efficiency of the current \primeTraveler{} for monitoring whether (and how much) its exploration task is held back by the presence/motion of the \secondaryTravelers{}. From now on we then denote this value as $\Lambda_p$, where
\[
p = i \quad s.t. \;\; \varFontNoSizeSmall{state}_i =\valueFontNoSizeSmall{prime-traveler}.
\]

This quantity is not in general locally available to every robot in the group, and therefore a simple decentralized algorithm is used for its propagation to avoid a flooding step. Among many possible choices we opted for using the following well-known consensus-based propagation~\citep{2003-OlfMur}:
\begin{align}
\begin{aligned}\label{eq:Lambda_estimate}
\dot{\hat\Lambda}_p^i &= k_\Lambda\sum_{j\in\mathcal{N}_i}({\hat\Lambda}_p^j - {\hat\Lambda}_p^i) \;\; &\text{if} \;\; i\neq p\phantom{.}\\
{\hat\Lambda}_p^i &= \Lambda_i \;\;&\text{if} \;\; i= p. 
\end{aligned}
\end{align}
This distributed estimator lets $\hat\Lambda_p^i$ track $\Lambda_p$ for all $i$ that hold $\varFontNoSizeSmall{state}_i\neq\valueFontNoSizeSmall{prime-traveler}$ with an accuracy depending on the chosen gain $k_\Lambda$. Notice that, for a constant $\Lambda_p$, the convergence of this estimation scheme is exact. Furthermore, since $\Lambda_p\in[0,1]$, ${\hat\Lambda}_p^i$ is then saturated so as to remain in the allowed interval despite the possible transient oscillations of the estimator. Instead of this simple consensus, one could also resort to a PI average consensus estimator~\citep{2006-FreYanLyn} to cope with presence of a time-varying signal. However, for simplicity we relied on a simple consensus law with 
less parameters to be tuned, and with, nevertheless, a satisfying performance as extensively shown in our simulation and experimental results.

Hence, every \secondaryTraveler{} can locally compute $\Theta_i$ and build an estimation $\hat\Lambda_p^i$ of $\Lambda_p$.
In order to consolidate these two quantities into a single value, we define the function \mbox{$\rho:[0,1]\times[0,1]\times[1,\infty)\to [0,1]$} as:
\begin{align}
 \rho(x,y,\sigma) = \left(1-x\right) y^\sigma + x\left(1-(1-y)^\sigma\right),
\end{align}
where $1\leq\sigma<\infty$ is a constant parameter. Gain $\rho_i$ is then obtained from $\Theta_i$ and $\hat\Lambda_p^i$ as
\begin{equation}\label{eq:gain}
\rho_i = \rho(\Theta_i,{\hat\Lambda}_p^i,\sigma)
\end{equation}
with $1\leq\sigma<\infty$ being a tunable parameter.

\begin{figure}[t]
\centering
{\includegraphics[width=0.49\columnwidth]{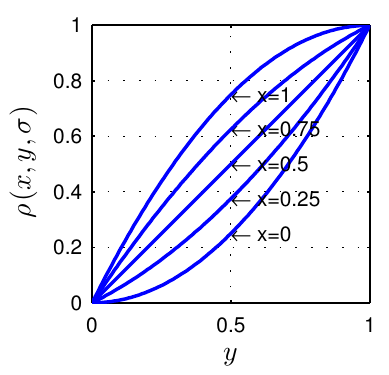}}{\includegraphics[width=0.49\columnwidth]{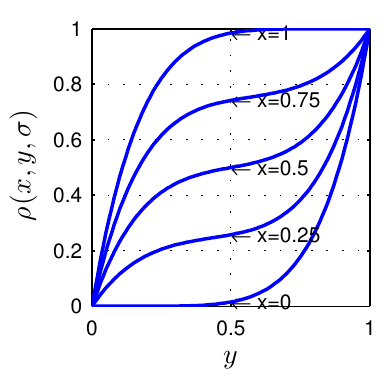}}\\
{\includegraphics[width=0.49\columnwidth]{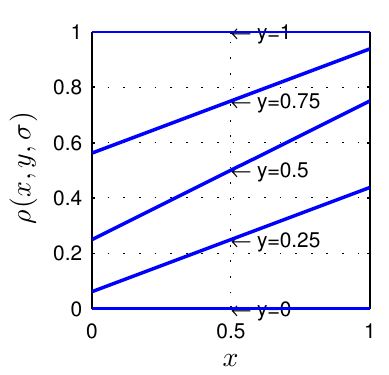}}{\includegraphics[width=0.49\columnwidth]{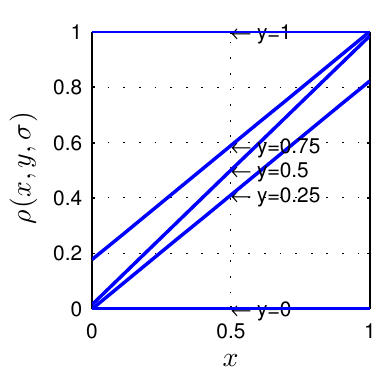}}\\
{\includegraphics[width=0.49\columnwidth]{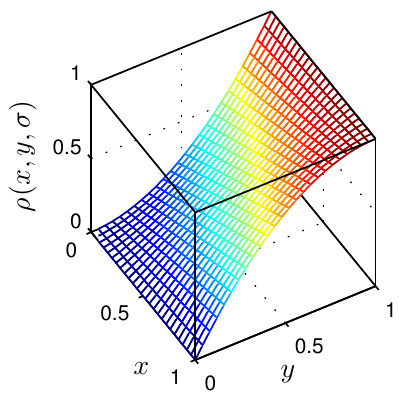}}{\includegraphics[width=0.49\columnwidth]{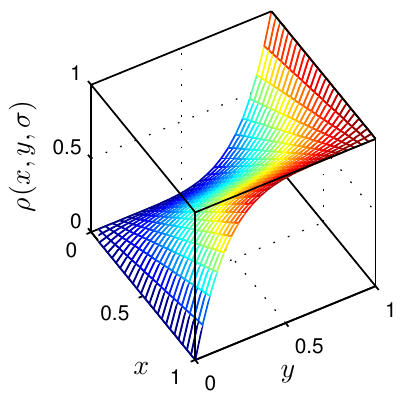}}\\
\caption[Function $\rho$ for two values of $\sigma$]{
Function $\rho$ for $\sigma=2$ (left side) and $\sigma=6$ (right side). The motion controller exploits this function by plugging the force direction alignment in the $x$ argument, and the estimate of the traveling efficiency of the current \primeTraveler{} in the $y$ argument.
}
\label{fig:rho}
\end{figure}

The reasons motivating this design of gain $\rho_i$ are as follows: $\rho_i$ is a smooth function of $\Theta_i$ and ${\hat\Lambda}_p^i$ possessing the following desired properties
(see also~\cref{fig:rho})
\begin{enumerate}
\item ${\hat\Lambda}_p^i=1$ $\Rightarrow$ $\rho_i=1$: if the traveling efficiency of the \primeTraveler{} is $1$ then every \secondaryTraveler{} sets $f_i=f_\text{travel}(q_i,\gamma_i,v_i^\text{cruise})$;
\item ${\hat\Lambda}_p^i=0$ $\Rightarrow$ $\rho_i=0$: if the traveling efficiency of the \primeTraveler{} is $0$ then every \secondaryTraveler{} sets $f_i=0$;
\item $\rho_i$ monotonically increases w.r.t. ${\hat\Lambda}_p^i$ for any $\Theta_i$ and $\sigma$ in their domains;
\item $\rho_i$ constantly increases w.r.t. $\Theta_i$ for any ${\hat\Lambda}_p^i\in(0,1)$ and $\sigma>1$; 
\item if $\sigma=1$ then $\rho_i={\hat\Lambda}_p^i$ for any $\Theta_i\in[0,1]$
\item if $\sigma\to\infty$ then $\rho_i\to\Theta_i$ for any ${\hat\Lambda}_p^i\in(0,1)$.
\end{enumerate}

Summarizing, gain $\rho_i$ mixes the information of both the force direction alignment and the traveling efficiency of the \primeTraveler{}, with more emphasis on the first or the second term depending on the value of the parameter $\sigma$. 
Nevertheless, the traveling efficiency ${\hat\Lambda}_p^i$ is always predominant at its boundary values ($0$ and $1$) regardless of the value of $\sigma$. 
This means that, whenever the estimated travel efficiency of the \primeTraveler{} is ${\hat\Lambda}_p^i=0$ and robot $i$ is a \secondaryTraveler{}, its \travelingForce{} is scaled to zero and, therefore, robot $i$ only becomes subject to the \connectivityAndDampingForce{}. Therefore, in this situation the motion of all \secondaryTravelers{} results dominated by the \primeTraveler{}, which is then able to execute its planned path towards its target location. On the other hand, when ${\hat\Lambda}_p^i=1$, the \primeTraveler{} has a sufficiently high traveling efficiency despite the \secondaryTraveler{} motions. Therefore, every \secondaryTraveler{} is free to travel along its own planned path regardless of the direction alignment between \travelingAndConnectivityForce{}.

We conclude noting that the main goal of the machinery defined in~\cref{sec:ctrlLoop}--\cref{sec:rho} is to ensure that the motion controller meets the requirements defined in~\cref{assumption:primeTravelerReachesTarget}. 
Although some of the steps involved in the design of the \travelingForce{} $f_i$ have a `heuristic' nature, the proposed algorithm is quite effective in solving the multi-target exploration task (in a decentralized way) under the constraint of connectivity maintenance, as proven by the several simulation and experimental results reported in the next section.

\section{Simulations and experiments}\label{sec:simExp}

In this section, we report the results of an extensive simulative and experimental campaign meant to illustrate and validate the proposed method. 
The videos of the simulations and experiments can be watched in the attached material and on
\url{http://homepages.laas.fr/afranchi/videos/multi_exp_conn.html}.

All the simulation (and experimental) results were run in 3D environments, although only a 2D perspective is reported in the videos for the simulated cases (therefore, robots that may look as `colliding' are actually flying at different heights, since their \connectivityForce{} prevents any possible inter-robot collision).

As robotic platform in both simulations and experiments we used small quadrotor UAVs (Unmanned Aerial Vehicles) with a diameter of 0.5\,m. This choice is motivated by the versatility and construction simplicity of these platforms, and also because of the good match with our assumption of being able to track any sufficiently smooth linear trajectory in 3D space.

We further made use of the SwarmSimX environment~\citep{2012m-LaeFraBueRob}, a physically-realistic simulation software. The simulated quadrotors are highly detailed models of the real quadrotors later employed in the experiments. The physical behavior of the robots itself and their interaction with the environment is simulated in real-time using PhysX\footnote{\url{http://www.geforce.com/hardware/technology/physx}}.

For the experiments, we opted for a highly customized version of the MK-Quadro\footnote{\url{http://www.mikrokopter.com}}. We implemented a software on the onboard microcontroller able to control the orientation of the robot by relying on the integrated inertial measurement unit. The desired orientation is provided via a serial connection by a position controller implemented within the ROS framework\footnote{\url{http://www.ros.org}} that can run on any generic GNU-Linux machine. The machine can be either mounted onboard or acting as a base-station. In the latter case a wireless serial connection with XBees\footnote{\url{http://www.digi.com/lp/xbee}} 
is used. We opted for the separate base station in order to extend the flight time thanks to the reduction of the onboard weight. The current UAV position used by the controller is retrieved from a motion capturing system\footnote{\url{http://www.vicon.com}}, while obstacles are defined statically before the task execution.

To abstract from simulations and experiments, we used the TeleKyb software framework, which is thoroughly described in~\cite{2013j-GraRieBueRobFra}. 
Finally, the desired trajectory (consisting of position, velocity and acceleration) is generated by our decentralized control algorithm implemented using Simulink\footnote{\url{http://www.mathworks.com/products/simulink/}} running in real-time at~$1$\,kHz.

\subsection{Monte Carlo Simulations}\label{sec:sim}


\def\figureSizeEmpty{0.433\textwidth}

\begin{figure*}
\hspace{-1.8cm}
\subfloat[\label{fig:empty:evo1}] {\includegraphics[angle=0, width=\figureSizeEmpty]{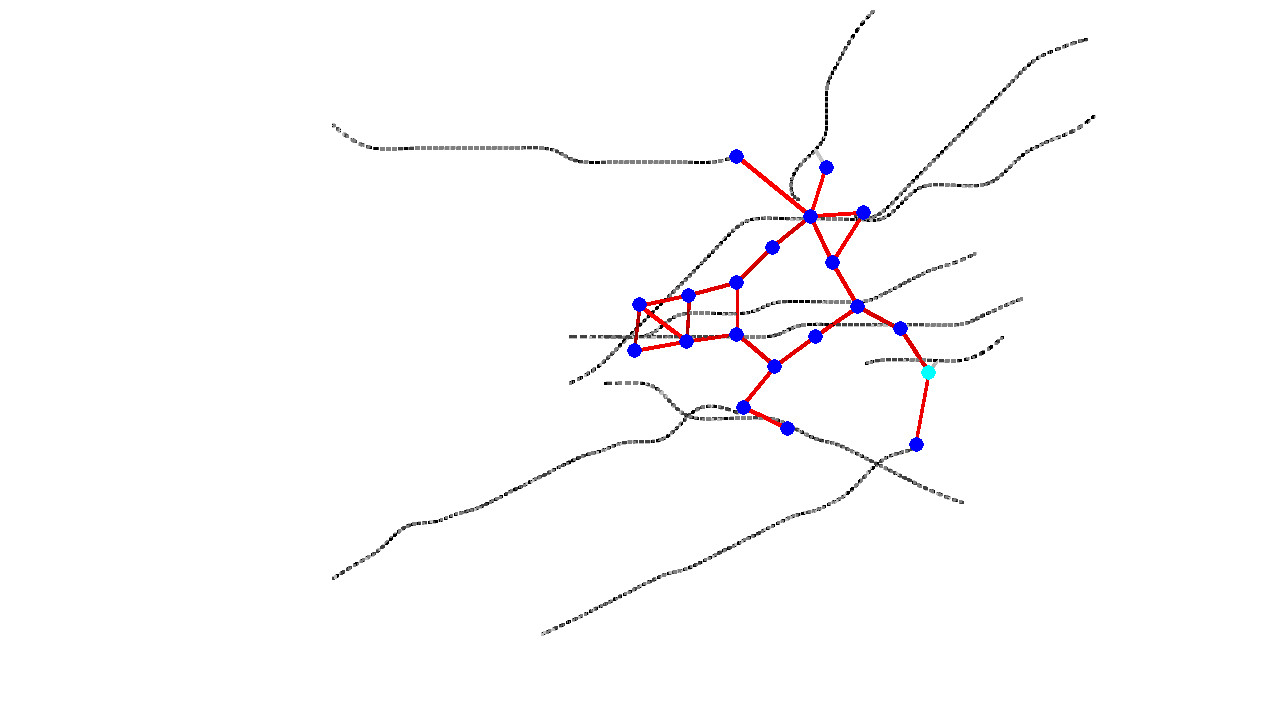}} 
\hspace{-1.1cm} 
\subfloat[\label{fig:empty:evo2}] {\includegraphics[angle=0, width=\figureSizeEmpty]{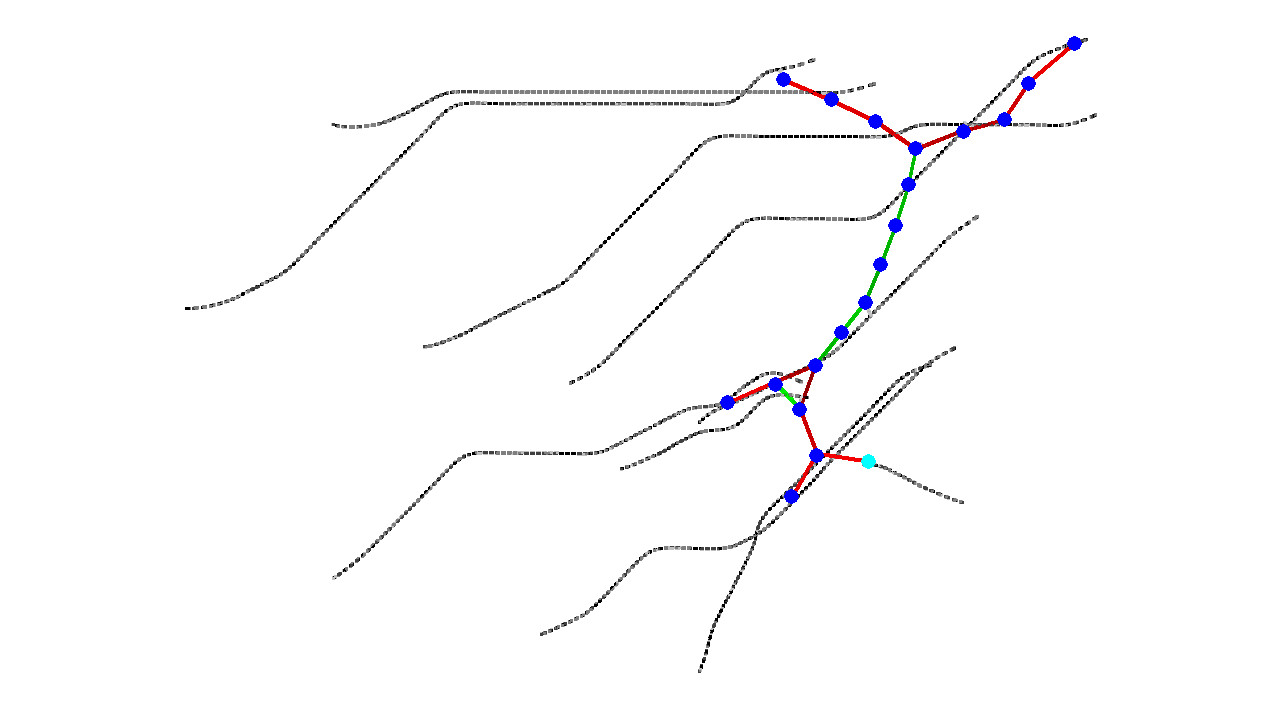}}
\hspace{-1.1cm} 
\subfloat[\label{fig:empty:evo3}] {\includegraphics[angle=0, width=\figureSizeEmpty]{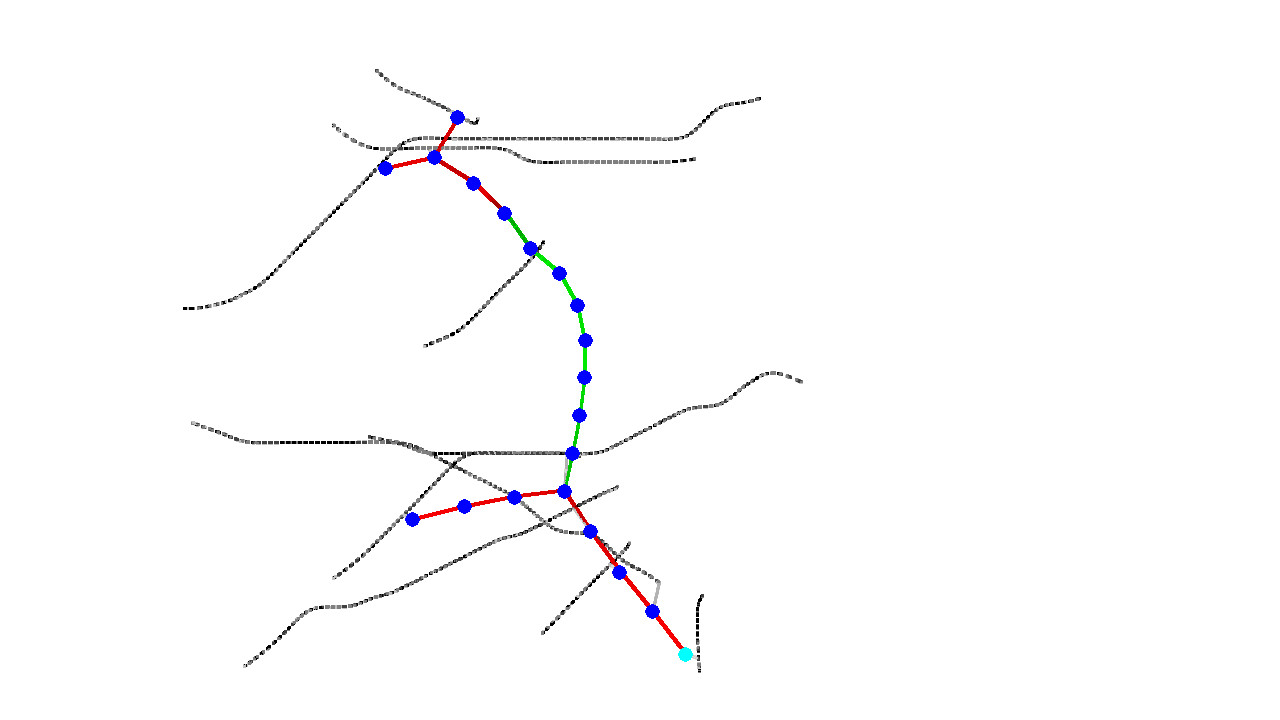}}
\caption{Snapshots of a simulation with 20 UAVs in empty space in three different consecutive time instants. The dotted black curves represent the planned path~$\gamma_i$ to the current target for each robot $i$ (if it has a current target). Blue dots are the robots, the turquoise dot is the current \primeTraveler{}. Line segments represent the presence of a connection link between a pair of robots with the following color coding: green -- well connected, red -- close to disconnection. The robots are able to concurrently explore the given targets and continuously maintain the connectivity of the interaction graph.
}
\label{fig:emptyEvo}
\end{figure*}



\def\figureSizeTown{0.331\textwidth}
\def\zoomed{_zoomed} 

\begin{figure*}
 \subfloat[\label{fig:town:evo1}]	{\includegraphics[angle=0, width=\figureSizeTown]{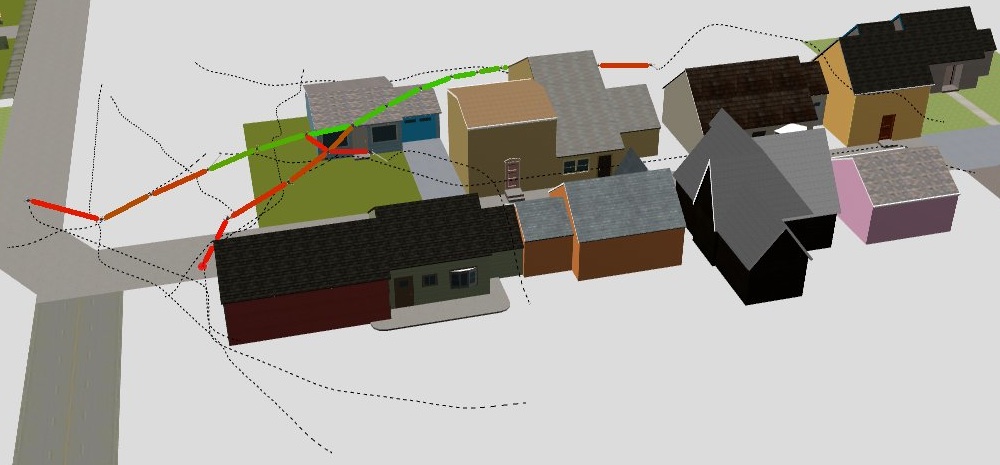}}\hfill
 \subfloat[\label{fig:town:evo2}]	{\includegraphics[angle=0, width=\figureSizeTown]{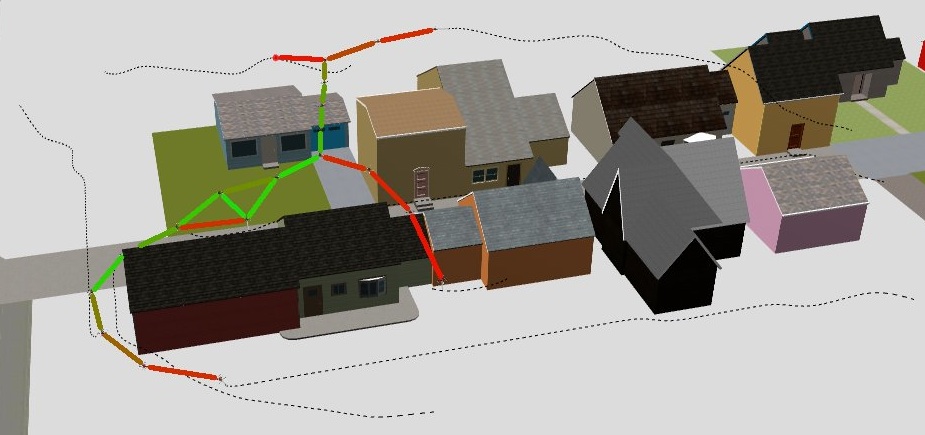}}\hfill
 \subfloat[\label{fig:town:evo3}]	{\includegraphics[angle=0, width=\figureSizeTown]{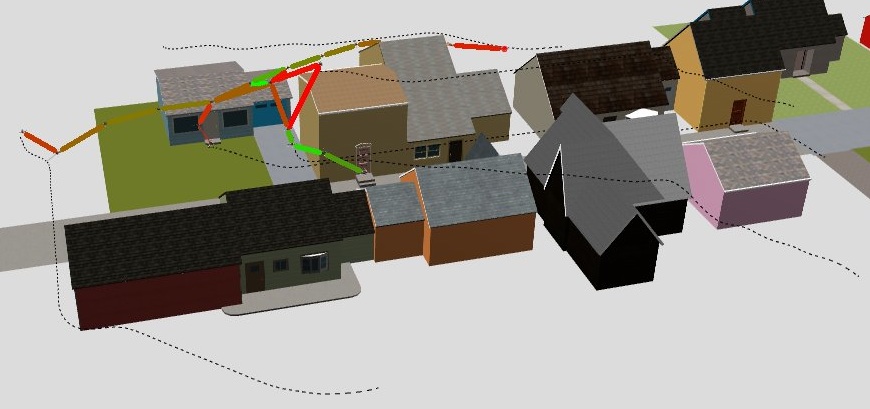}}
\caption{Three snapshots of consecutive time instants of a simulation in the town environment. Graphical notation similar to~\cref{fig:emptyEvo}}
\label{fig:townEvo}
\end{figure*}



\def\fileFormat{jpg}

\def\WhiteBalance{WhiteBalance} 

\def\figureSizeOffice{0.333\textwidth}
\begin{figure*}
\centering
 \subfloat[\label{fig:office:evo14}]	{\includegraphics[angle=0, width=\figureSizeOffice]{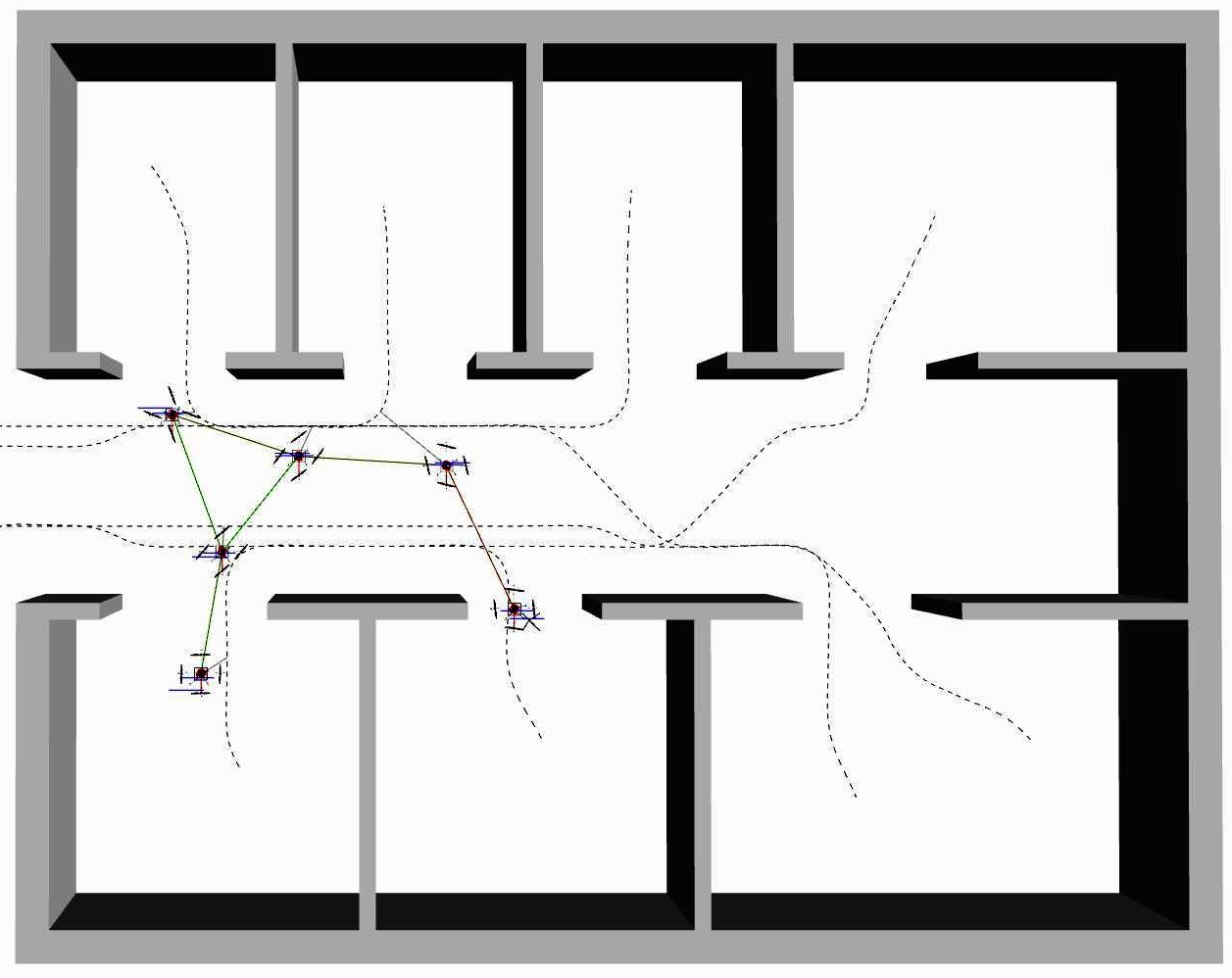}}\hfill
 \subfloat[\label{fig:office:evo24}]	{\includegraphics[angle=0, width=\figureSizeOffice]{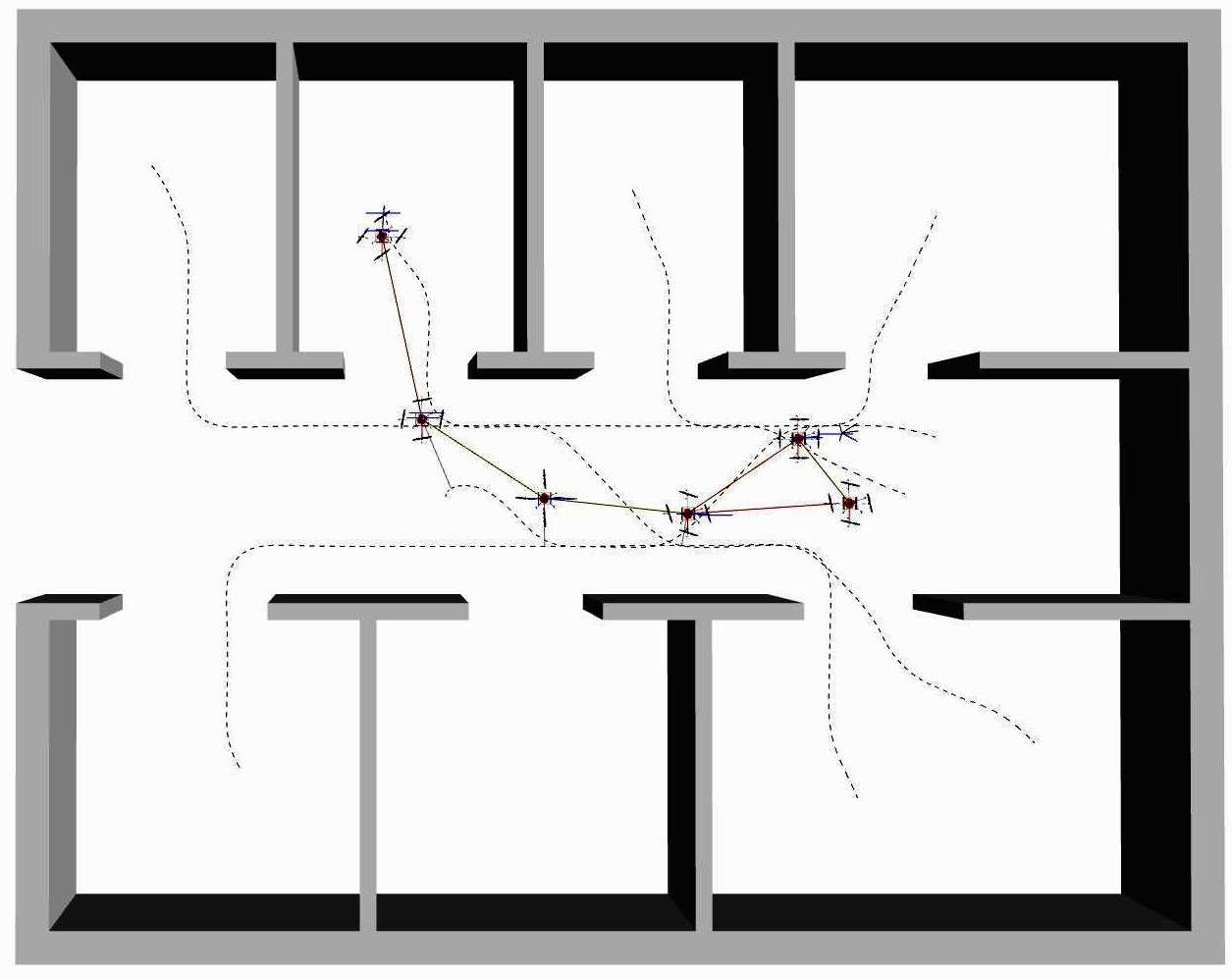}}\hfill
 \subfloat[\label{fig:office:evo34}]	{\includegraphics[angle=0, width=\figureSizeOffice]{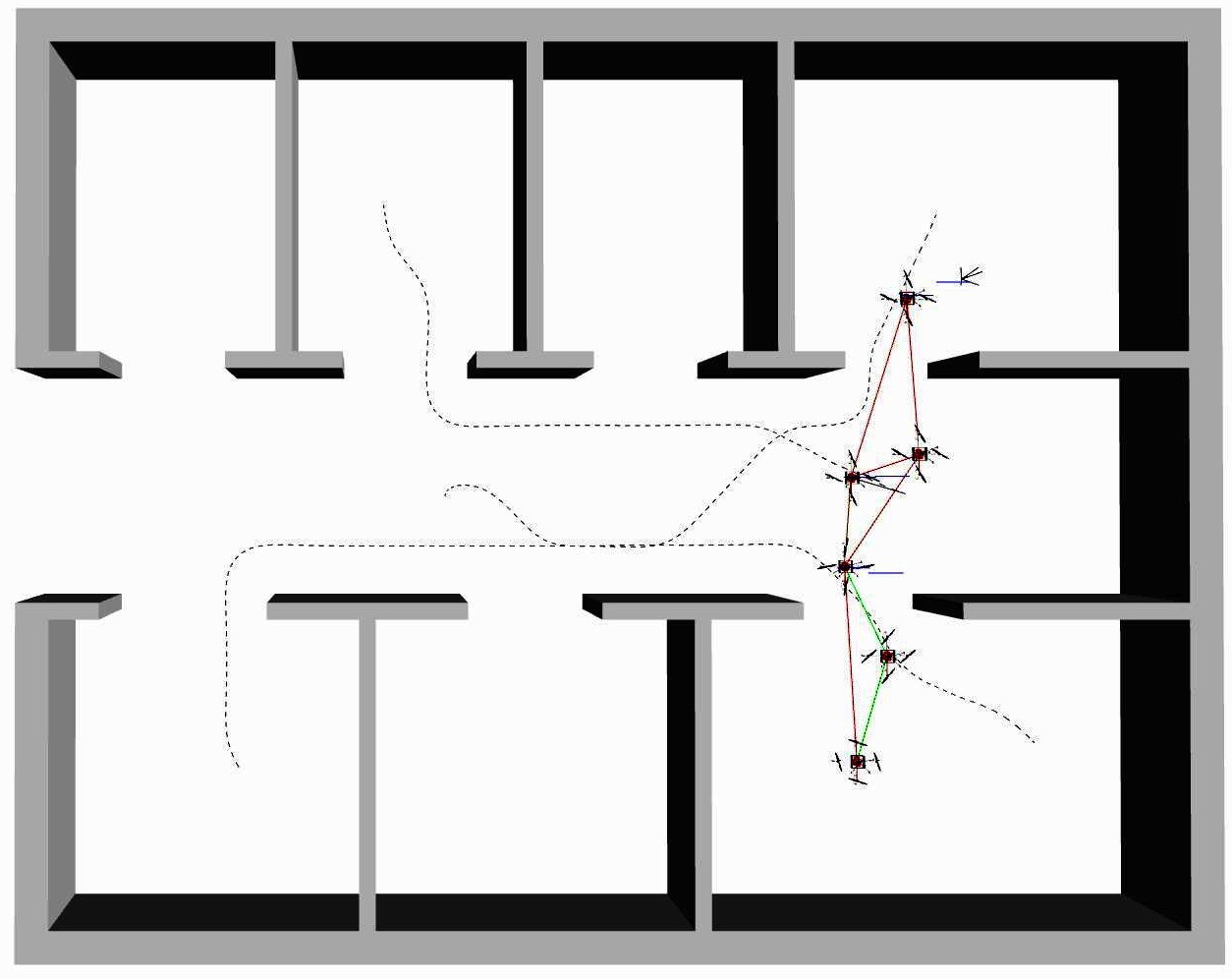}}
\caption{Snapshots of a simulation in the office-like environment in three consecutive time instants. Graphical notation similar to~\cref{fig:emptyEvo}}
\label{fig:officeEvo}
\end{figure*}

The proposed method has been extensively evaluated through randomized experiments in three significantly different scenarios. The first scenario is an obstacle free 3D space and three snapshots of the evolution of the proposed algorithm are presented in~\cref{fig:emptyEvo}. The second, a more complex, scenario includes a part of a town and is reported in~\cref{fig:townEvo}. 
The third is an office-like environment shown in~\cref{fig:officeEvo}. The size of the environments is $50\text{\,m}\times 70\text{\,m}$ for both the empty space and the town, and about $10\text{\,m}\times 15\text{\,m}$ for the office. Since the first two environments are outdoor scenarios and the office-like environment is indoor, two different sets of parameters were employed in the simulations. The values of the main parameters are listed in~\cref{tab:parametersSimulations}.


\begin{table}[hb]
\caption{Main parameters of the algorithm used in the 1800 randomized simulative trials in the different scenarios.}
\label{tab:parametersSimulations}
\centering
\begin{tabular}{c|c|c}
parameter											& empty space and town							& office											\\ \hline
$(R_s',R_s)$										& ($2.5$\,m,\ $6$\,m)							& ($1.1$\,m,\ $2.5$\,m)								\\
$(R_o,R_o')$										& ($0.75$\,m,\ $1.75$\,m)						& ($0.25$\,m,\ $0.6$\,m)							\\
$(R_c,R_c')$										& ($1$\,m,\ $2.5$\,m)							& ($0.8$\,m,\ $1.1$\,m)								\\
$(\lambda_2^{\text{min}}, \lambda_2^{\text{null}})$	& $(0,1)$										& $(0,1)$											\\ \hline
$R_\text{grid}$										& $0.75$\,m										& $0.25$\,m											\\
$\sigma$											& $3$											& $3$												\\
$v_i^\text{cruise}$									& $3$\,m/s										& $1$\,m/s for all $i$								\\
$(x_c,x_M)$                     & $(0.1, 0.6) v_i^\text{cruise}$  & $(0.1, 0.6) v_i^\text{cruise}$    \\
$\Delta t_i^k$										& $3$\,s for all $i$ and $k$					& $3$\,s for all $i$ and $k$						\\
$R_z$												& $1.8$\,m										& $1$\,m
\end{tabular}
\end{table}



The number of robots varied from $10$ to $35$. In every trial $3$ targets are sequentially assigned to $5$ robots and $2$ targets are sequentially assigned to other $5$ robots, for a total of $25$ targets per trial. The remaining robots are given no targets (i.e., they act always as \connectors{}).

The configuration of the given targets is randomized across the different trials. The same random configurations are repeated for every different number of robots in order to allow for a fair comparison among the results. In the following we refer to the robots with at least one target assigned during a trial as \explorers{}.

To summarize, we simulated a total number of $1800$ trials arranged in the following way: in each of the $3$ scenes, and for each of the $100$ target configurations in each scene, we ran a simulation with $6$ different numbers of robots, namely $10$, $15$, $20$, $25$, $30$, $35$. We encourage the reader to also watch the attached video where some representative simulative trials are shown. 

\begin{figure*}
\centering
\includegraphics[width=0.836\textwidth]{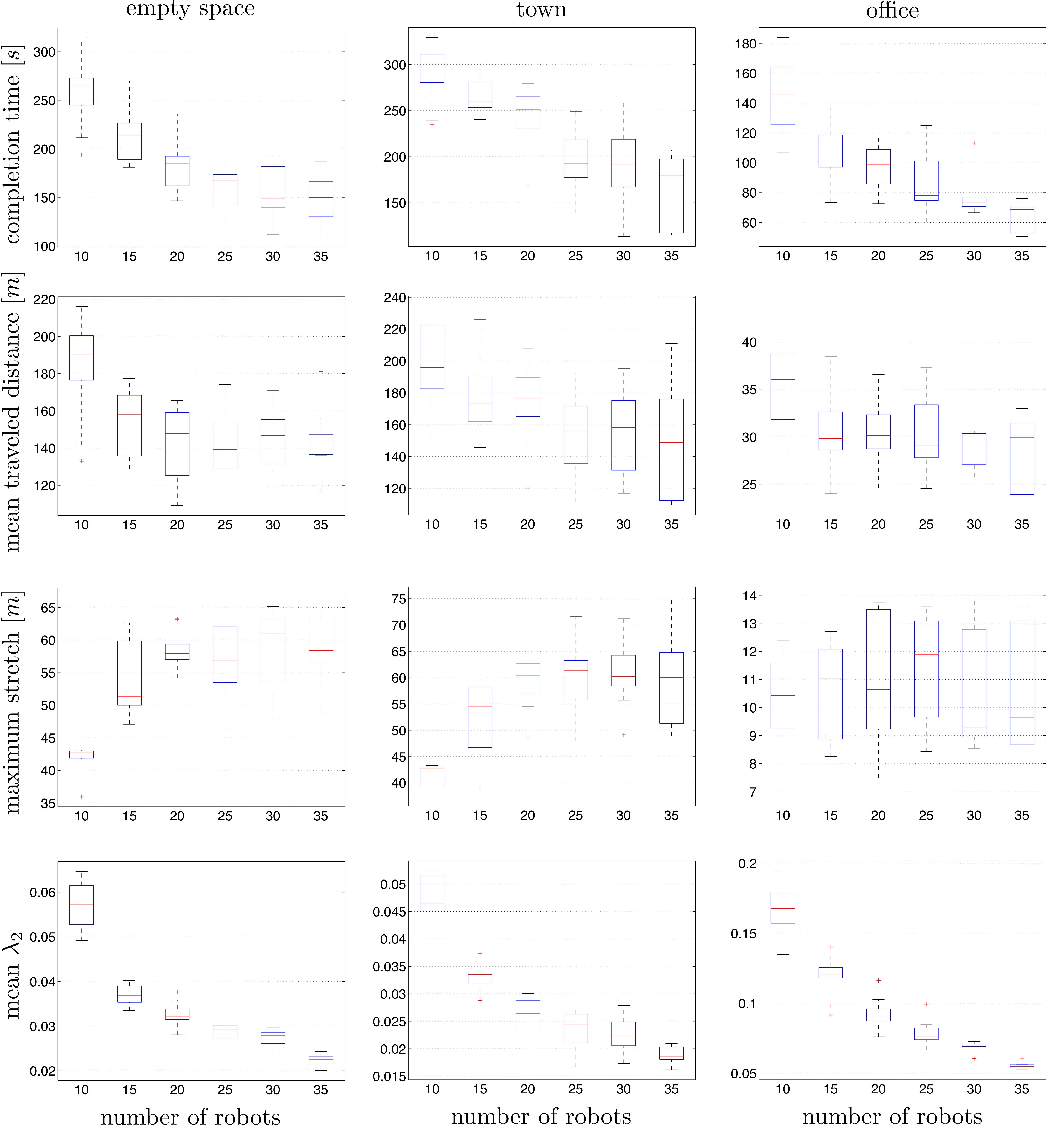}
\caption{Statistics of the completion times (first row), mean traveled distance of the traveling robots (second row), the maximum Euclidian distance between two traveling robots (third row) and mean $\lambda_2$ (forth row) versus the number of robots in the environments empty space (left column), town (middle column) and office (right column).}
\label{fig:statistics}
\end{figure*}

In~\cref{fig:statistics} we show the evolutions of the statistical percentiles of:
\begin{itemize}
\item the overall completion time, 
\item the mean traveled distance of the $10$ \explorers{}, 
\item the maximum Euclidean distance between two \explorers{}
\item the average of $\lambda_2(t)$ over time along the whole trial (we recall that the larger the $\lambda_2$ the more connected is the group of robots, refer to Appendix~\ref{app:connectivity}),
\end{itemize}
when the number of robots varies from $10$ (i.e., no \connectors{}) to $35$ (i.e., $25$ \connectors{}). Each column refers to one of the $3$ different scenarios.

An improvement with the increasing number of \connectors{} in all scenarios is obvious. The mean completion time (first row) roughly halves when comparing $0$ to $25$ \connectors{}.
Adding more than $25$ connectors will likely produce only minor improvement compared to the higher cost of having more robots, since the trend becomes basically flat. For this reason we did not perform simulations with a larger number of robots.

In the second row (mean traveled distance) one can see how, by already adding a few robots, a reduced mean traveled distance is obtained. This can be explained by the fact that the \connectors{} make the \explorers{} less disturbed by other \explorers{} with, therefore, more freedom to avoid unnecessary detours in reaching their targets. 

Another measure of the reduced task completion time is the maximum stretch among the \explorers{} (i.e., the maximum Euclidean distance between any two \explorers{}, see third row). The more connectors, the more stretch is allowed: \connectors{} in fact provide the support needed by the \explorers{} for keeping graph $\calG$ connected while freely moving towards their targets.
Only the office-like environment does not show this trend in the maximum stretch. This is due to the fact that the scene is relatively small and therefore the targets are not enough spread apart, so no bigger stretch is needed. 

The increased freedom of the \explorers{} is also evident in the plots of the average $\lambda_2(t)$ (fourth row). These plots show how the \connectors{} are also useful to let the \explorers{} move more freely even in small environments. In fact, the larger the amount of \connectors{}, the lower the mean $\lambda_2$: with more connectors the \explorers{} are more able to simultaneously travel towards their targets, thus bringing the topology of the group closer to less connected topologies (i.e., closer to tree-like topologies where the explorers would be the leaves of the tree). 
Clearly, this effect is independent of the maximum stretch, in fact the average $\lambda_2$ follows this decreasing trend also in the third office-like environment (third column).

\subsection{Experiments}\label{sec:exp}

The experiments involved $6$ real quadrotors and were meant to test the applicability of the algorithm in a real scenario. 
The parameters of the algorithm used in the experiments are reported in~\cref{tab:parametersExperiments}. 


\begin{table}[b]
\caption{Main parameters used in the experiments.}
\label{tab:parametersExperiments}
\centering
\begin{tabular}{c|c}
parameter												& value															\\ \hline
$(R_s',R_s)$											& ($1.4$\,m,\ $2.5$\,m)											\\
$(R_o,R_o')$											& ($0.5$\,m,\ $0.75$\,m)										\\
$(R_c,R_c')$											& ($1.0$\,m,\ $1.4$\,m)											\\
$(\lambda_2^{\text{min}}, \lambda_2^{\text{null}})$		& $(0,1)$														\\ \hline
$R_\text{grid}$											& $0.2$\,m														\\
$\sigma$												& $3$															\\
$v_i^\text{cruise}$										& $0.5$\,m/s													\\
$(x_c,x_M)$												& $\left(0.2, 0.7\right) v_i^\text{cruise}$						\\
$\Delta t_i^k$											& $3$\,s for all $i$ and $k$									\\
$R_z$													& $0.75$\,m
\end{tabular}
\end{table}



In order to obtain a $\bar{\mathcal{C}}^4$ trajectory smoother than~$q_i(t)$ and, thus, better matching the dynamics capabilities of a quadrotor UAV~\citep{2001-MisBenMsi}, we made use of a fourth order linear filter for each quadrotor:
\begin{equation}\label{eq:filter}
    \ddddot{q}^f_i(t) = -k_1\dddot{q}^f_i(t) -k_2 \ddot{q}^f_i(t) - k_3\dot{q}^f_i(t) + k_4(q_i(t)-q^f_i(t))
\end{equation}
that tracks the position of the original trajectory $q_i(t)$, while keeping the velocity, acceleration, and jerk low in the filtered trajectory.
The tunable gains were chosen as~$k_1=44$, $k_2=707$, $k_3=5090$, $k_4=13692$ for placing the (real negative) poles at approximately $-12$, $-13$, $-14$, $-15$, then resulting in a settling time of about $0.3$\,s within a band of $5\%$.

The resulting trajectory $q^f_i(t)$ is then provided in place of $q_i(t)$ as input trajectory for the robot $i$ as defined in~\cref{eq:second_order}, since it results very close to $q_i(t)$ as shown in~\cref{fig:exp:smoothPos}. However, at the same time, it provides a much smoother reference position signal to the quadrotor by filtering off occasional abrupt motions, as can be seen in the velocity and acceleration reported in~\cref{fig:exp:smoothVel,fig:exp:smoothAcc}.
\Cref{fig:experimentPlots:posError} shows the norms of the UAV errors while tracking the desired trajectory~$q^f_i(t)$. The average norm of all the quadrotors tracking errors during the whole experiment is $0.021$\,m, a few short peaks are above $0.06$\,m, and the highest peak is about $0.098$\,m.

\begin{figure}[t]
\centering
\subfloat[\label{fig:exp:smoothPos}]	{\includegraphics[width=0.9\columnwidth]{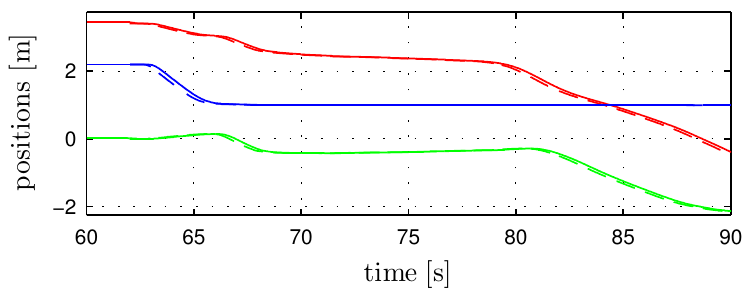}}
\\
\vspace{-0.9em}
\subfloat[\label{fig:exp:smoothVel}]	{\includegraphics[width=0.9\columnwidth]{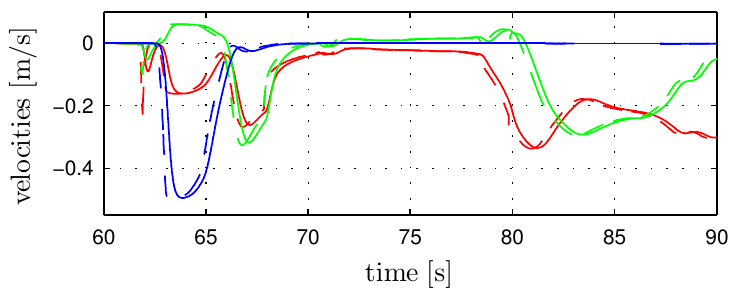}}
\vspace{-0.9em}
\\
\vspace{-0.9em}
\subfloat[\label{fig:exp:smoothAcc}]	{\includegraphics[width=0.9\columnwidth]{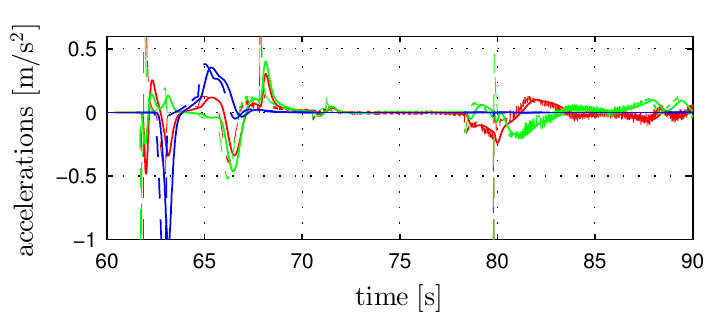}}
\vspace{-0.3em}
\caption{Position, velocity and acceleration of~\cameraRobot{} during a representative period of the experiment, where $q_i(t), \dot{q}_i(t), \ddot{q}_i(t)$ are plotted in dash and $q_i^f(t), \dot{q}_i^f(t), \ddot{q}_i^f(t)$ as solid curves. The $x$, $y$ and $z$ component is plotted in red, green and blue respectively.}
\label{fig:experimentPlots:smoothingAccVelPos}
\end{figure}
\begin{figure}
\centering
\includegraphics[width=0.9\columnwidth]{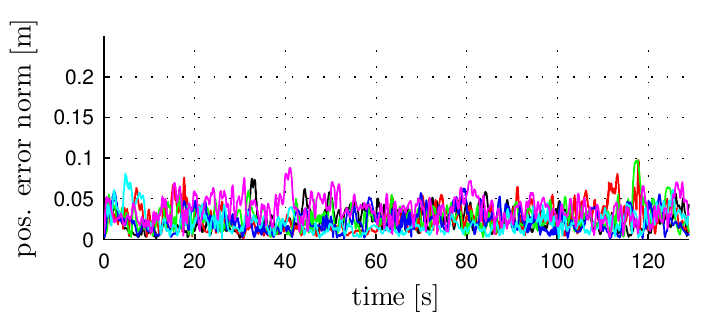}
\caption{Plots of the $6$ norms of the position error between $q_i(t)$ and the corresponding real quadrotor trajectory, for $i=1,\ldots,6$. The average error norm is 0.021\,m.}
\label{fig:experimentPlots:posError}
\end{figure}

For these experiments, we reproduced a scene similar to the office-like environment used in simulation, see~\cref{fig:experimentPaths}. The UAVs with IDs `$2$' and `$4$' (called \explorers{}) were given some targets, while the UAVs with IDs `$1$', `$3$', `$5$', and `$6$' (\connectors{}) had no target, for then a total of $6$ quadrotors.

The \emph{\cameraRobot{}} (with ID `4') carries an onboard camera and has two targets in total. Whenever it reaches one of its targets it gives a human operator direct control of the vehicle in the surrounding area of the target.
Then, with the help of the onboard camera, the human operator has the task of searching for an object in the environment. When the object is found by the human operator, the task at the target is considered completed, and the UAV switches back to autonomous control.
In order to allow full human control of \cameraRobot{} in the anchoring behavior, the UAV is temporarily decoupled from the point $q_4$, which is instead kept close to the target by the action of $f_\text{anchor}$ (as desired). 
The \emph{\pickNplaceRobot{}} (with ID `2') is instead fully autonomous and is assigned with a total of $4$ targets. At the first target location, the task is to pick up an object to then be released at the second target location. The same task is subsequently repeated with targets $3$ and $4$. We note, however, that the pick and place action is only virtually performed since the employed quadrotors are not equipped with an onboard gripper.
We also stress that all these operations are performed concurrently while keeping the topology of the group connected at all times.

\begin{table}[t]
\caption{Chronological list of important events in the experiment.}
\label{tab:chronological_events}
\centering
\begin{tabular}{|c|c|p{0.623\columnwidth}|}
\hline
\textbf{\cref{fig:experimentPaths}} & \textbf{Time} & \hfill \textbf{Events} \hfill~ \\
\hline
\multirow{7}{*}{\cref{fig:experimentPaths:1}} & 0\,s  &	The experiment starts. Both \explorers{} 
are assigned a target and since \pickNplaceRobot{} is closer to its goal, it becomes \primeTraveler{}, while \cameraRobot{} is \secondaryTraveler{}. \\
& 22\,s  &	\PickNplaceRobot{} arrives at its first target, where it should pick up an object. Therefore \pickNplaceRobot{} goes into \anchor{} and \cameraRobot{} becomes \primeTraveler{}. \\
\hline
\multirow{10}{*}{\cref{fig:experimentPaths:2}} & 29\,s  &	\PickNplaceRobot{} has completed the pick-up action and receives the point to release the object as a new target. Since \cameraRobot{} is still \primeTraveler{}, \pickNplaceRobot{} becomes \secondaryTraveler{}. \\
& 35\,s  &	\CameraRobot{} arrives at its target, where the human operator takes control of the UAV and use its camera to find a yellow picture on the wall. \PickNplaceRobot{} then becomes \primeTraveler{}. \\
& 56\,s  &	\PickNplaceRobot{} arrives at the target where it needs to release the object. \\ 
\hline
\multirow{12}{*}{\cref{fig:experimentPaths:3}} & 63\,s  &	\PickNplaceRobot{} has completed the releasing action and receives the next pick-up location. \CameraRobot{} is still under the control of the human operator and therefore in an \anchor{} state, so \pickNplaceRobot{} directly becomes \primeTraveler{}. \\
& 65\,s  &	The human operator finds the picture on the wall, \cameraRobot{} becomes autonomous again and starts to move towards its next target as \secondaryTraveler{}, since \pickNplaceRobot{} is \primeTraveler{}. \\
& 78\,s  &	\PickNplaceRobot{} arrives at the location where to pick up the second object and goes to the \anchor{} state. Hence \cameraRobot{} becomes \primeTraveler{}. \\
\hline 
\multirow{8}{*}{\cref{fig:experimentPaths:4}} & 85\,s  &	\PickNplaceRobot{} has completed the pick-up action and starts moving towards the releasing location as \secondaryTraveler{}. \\
& 100\,s  &	\cameraRobot{} arrives at its target, goes to \anchor{} state and is thus under control of the human operator, therefore \pickNplaceRobot{} becomes \primeTraveler{}. \\
& 119\,s  & \PickNplaceRobot{} arrives at its final target and switches into \anchor{}. \\
\hline
\multirow{8}{*}{\cref{fig:experimentPaths:5}} & 123\,s  & The human operator finds the searched object and \cameraRobot{} becomes \connector{} since it has no new target location. \\
& 126\,s  & \PickNplaceRobot{} has completed the releasing action and becomes a \connector{} since it has also no new target. \\
& 129\,s  & No UAV has a next target and the experiment ends.\\
\hline
\end{tabular}
\end{table}

A video of the experiment is present in the attached material and can be found under the given link above.

\Cref{tab:chronological_events} reports and describes all the relevant events taking place during an experiment in a chronological order. 


\Cref{fig:experimentPaths} shows the top-view of the \explorer{} paths for five representative time periods: $T_1=[0, 25]$\,s in \cref{fig:experimentPaths:1}, $T_2=[25, 60]$\,s in \cref{fig:experimentPaths:2}, $T_3=[60, 80]$\,s in \cref{fig:experimentPaths:3}, $T_4=[80, 120]$\,s in \cref{fig:experimentPaths:4}, and finally $T_5=[120, 129]$\,s in \cref{fig:experimentPaths:5}.
Every plot shows the (connected) graph topology of the group at the beginning of the time interval (dashed black lines) and the paths of the 2 \explorers{} (solid lines, blue for the \cameraRobot{} and red for the \pickNplaceRobot{}).
The initial positions of the robots are shown with colored circles and are labeled with the IDs of the corresponding robots.
The two small blue squares represent the two desired target locations of the \cameraRobot{}. The two green squares and the two red squares represent the two pick positions and release positions of \pickNplaceRobot{}, respectively. Finally, the vertical walls of the environment are shown in gray.
\Cref{fig:experimentPaths:z} on the other hand shows the $z$-coordinate of all the six quadrotors in order to understand the 3D motion in the 2D projections of \cref{fig:experimentPaths:1,fig:experimentPaths:2,fig:experimentPaths:3,fig:experimentPaths:4,fig:experimentPaths:5}.

\begin{figure}
\centering
\subfloat[\label{fig:experimentPaths:1}]{\includegraphics[width=0.479\columnwidth]{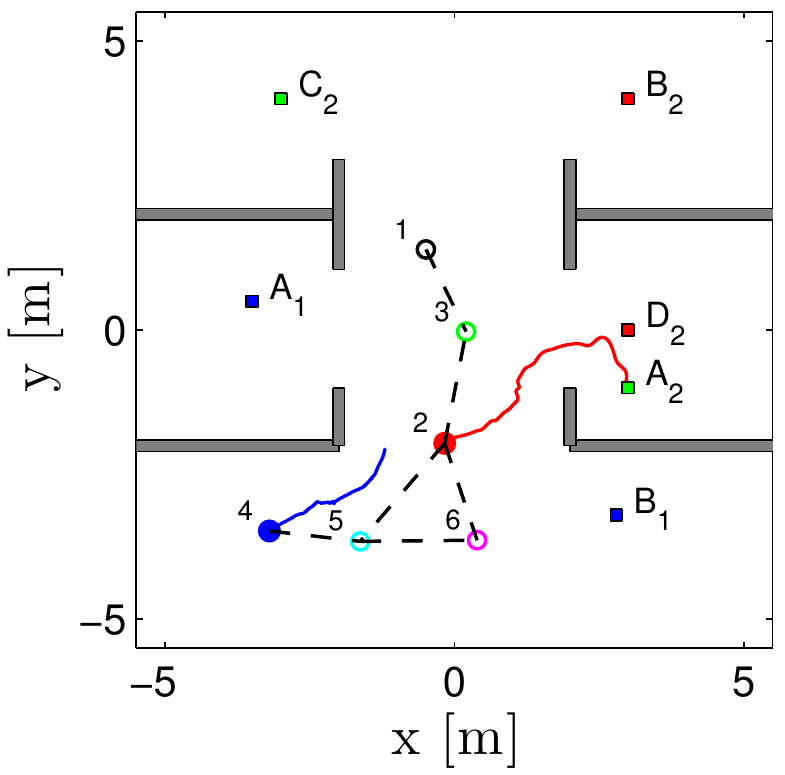}}\hfill
\subfloat[\label{fig:experimentPaths:2}]{\includegraphics[width=0.479\columnwidth]{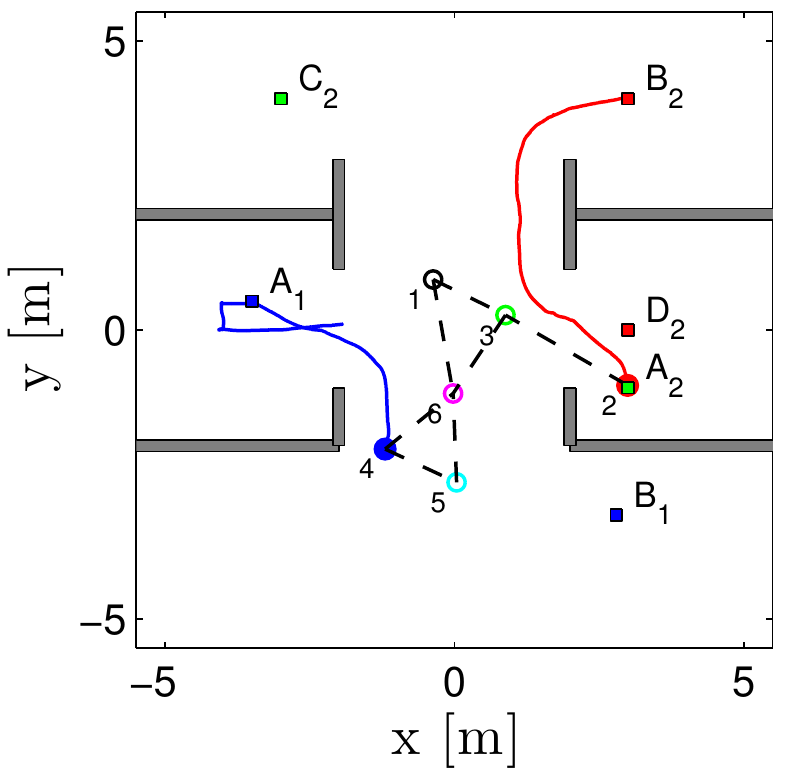}}\\
\subfloat[\label{fig:experimentPaths:3}]{\includegraphics[width=0.479\columnwidth]{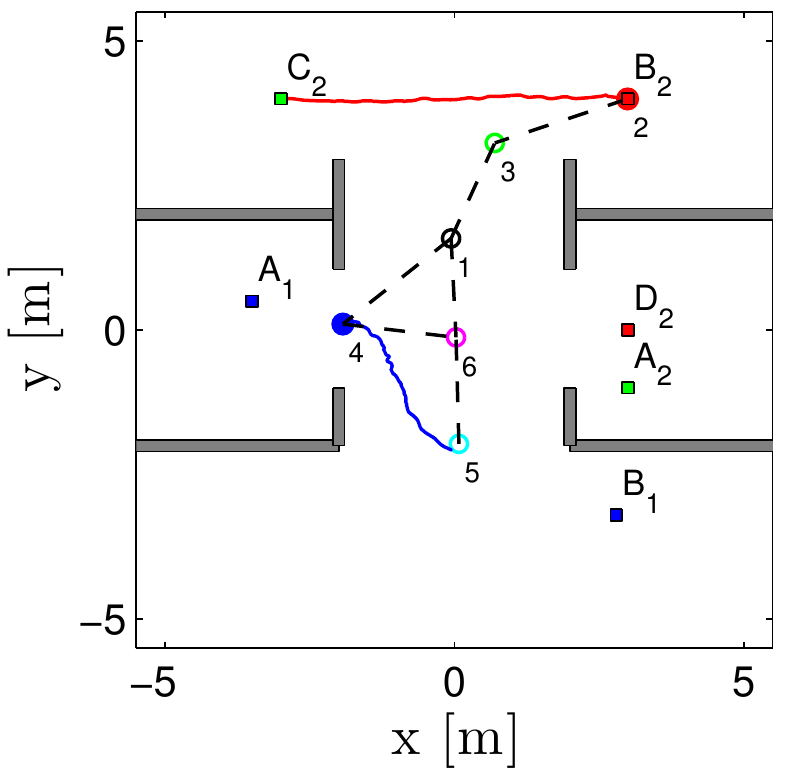}}\hfill
\subfloat[\label{fig:experimentPaths:4}]{\includegraphics[width=0.479\columnwidth]{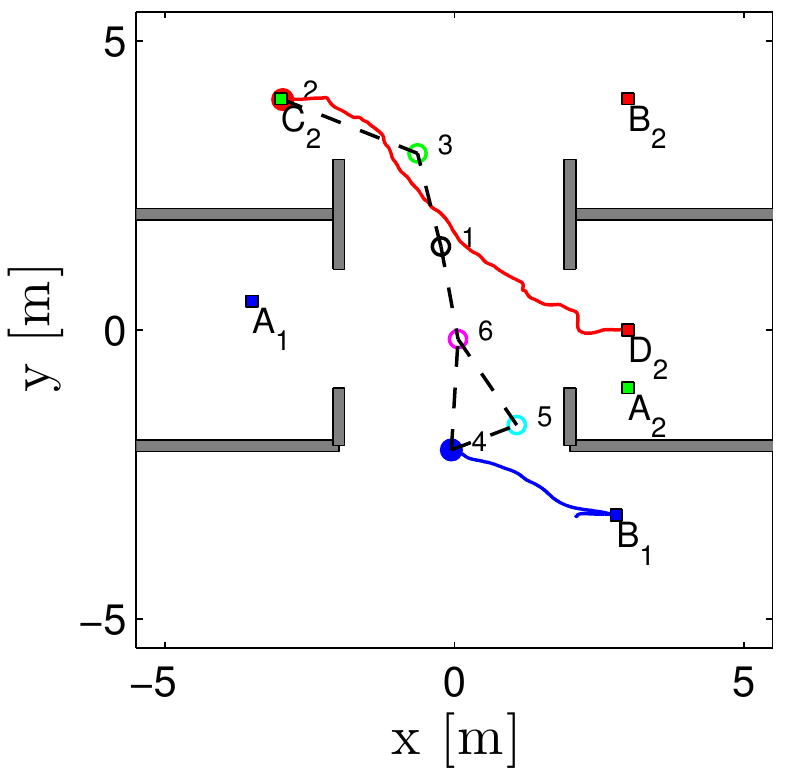}}\\
\subfloat[\label{fig:experimentPaths:5}]{\includegraphics[width=0.479\columnwidth]{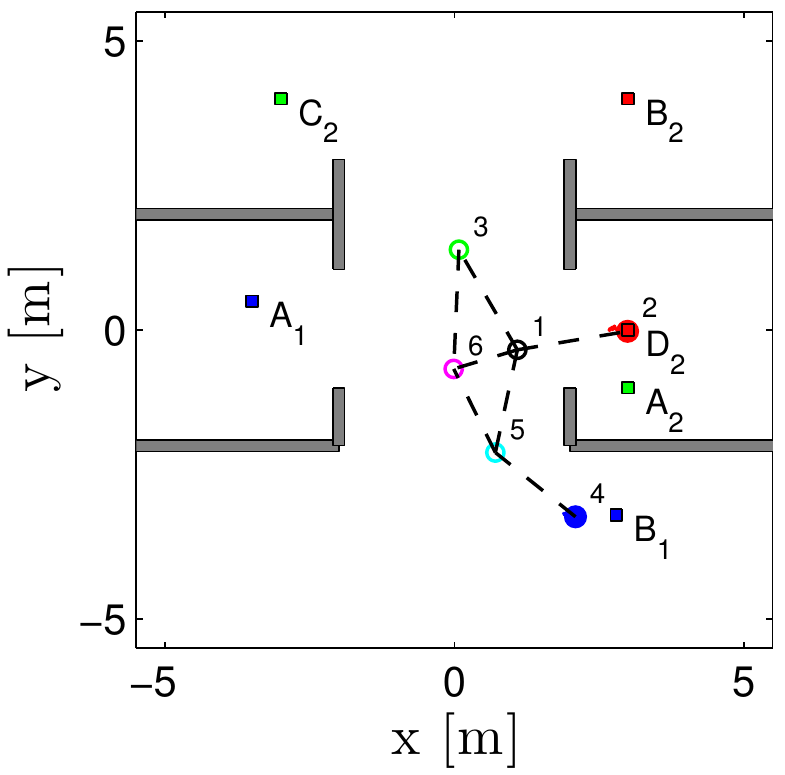}}\hfill
\subfloat[\label{fig:experimentPaths:z}]{\includegraphics[width=0.479\columnwidth]{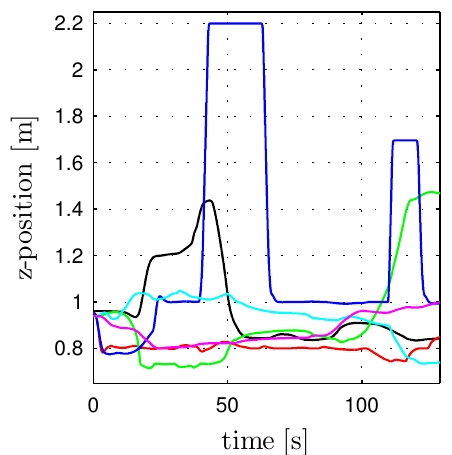}}
\caption{\mySubref{fig:experimentPaths:1}-\mySubref{fig:experimentPaths:5} Top view of the 3D paths of the \explorers{} (solid blue and red curves) during the experiment in five representative time intervals.
The interaction graph at the beginning of each interval is shown with black dashed lines. 
The ID of each robot is shown besides the circle representing the starting position of each robot at the beginning of the corresponding interval. Targets are represented with colored squares and walls are gray. The specific time intervals are: 
\mySubref{fig:experimentPaths:1}~$T_1=[0, 25]$\,s, \mySubref{fig:experimentPaths:2}~$T_2=[25, 60]$\,s, \mySubref{fig:experimentPaths:3}~$T_3=[60, 80]$\,s, \mySubref{fig:experimentPaths:4}~$T_4=[80, 120]$\,s and \mySubref{fig:experimentPaths:5}~$T_5=[120, 129]$\,s. \mySubref{fig:experimentPaths:z}~$z$-coordinate of the positions of all six quadrotors to help interpreting the 2D projection reported in the plots (and videos). The large vertical motion of \cameraRobot{} (blue) is due to the human operator flying this robot, while the subsequent descent is autonomously performed thanks to the proposed algorithm.}
\label{fig:experimentPaths}
\end{figure}

\begin{figure*}
\centering
\subfloat[\label{fig:experimentScreenshots:side}]		{\includegraphics[height=0.425\columnwidth]{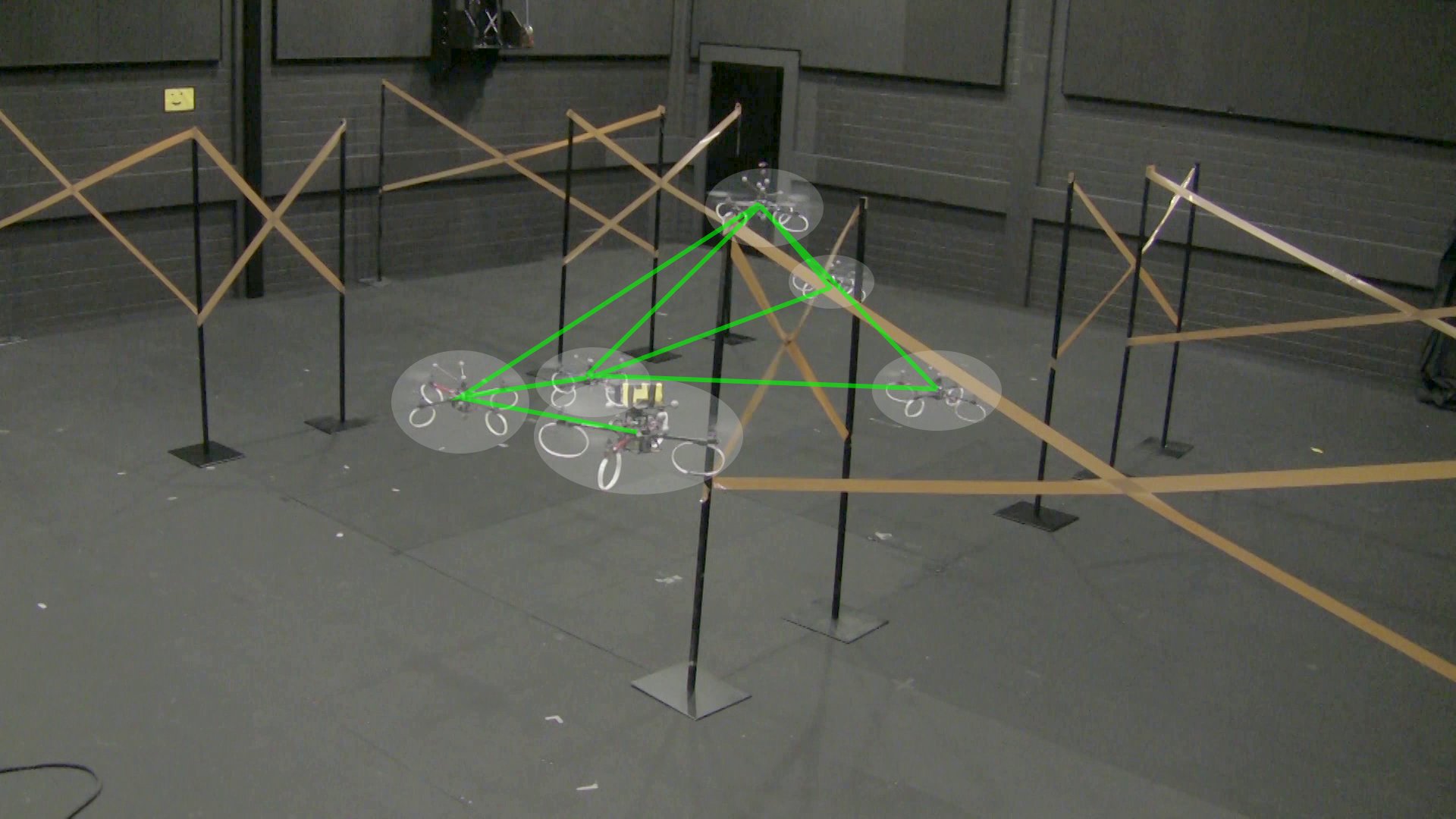}}\hfill
\subfloat[\label{fig:experimentScreenshots:onboard}]	{\includegraphics[height=0.425\columnwidth]{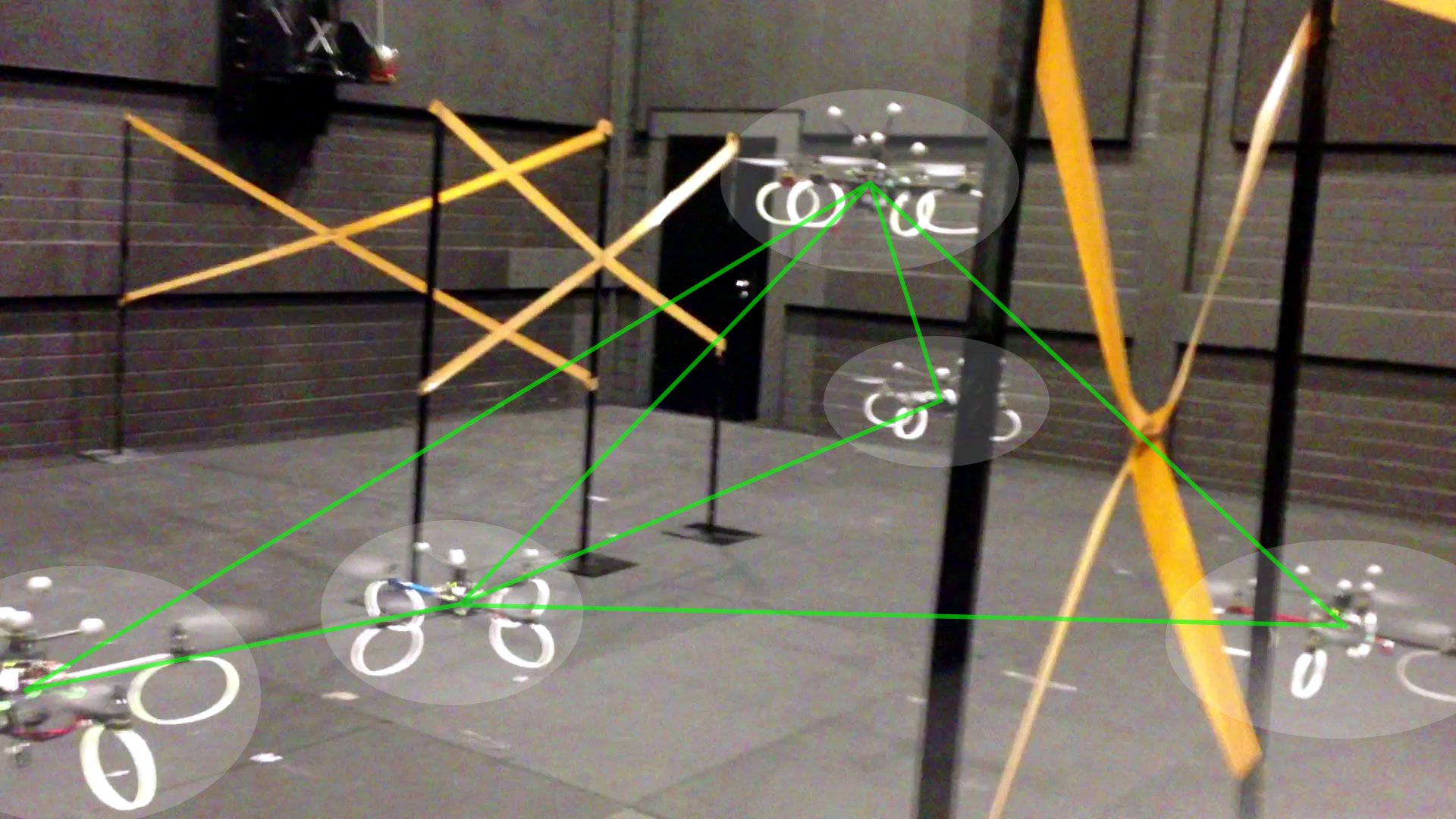}}\hfill
\subfloat[\label{fig:experimentScreenshots:SSX}]		{\includegraphics[height=0.425\columnwidth]{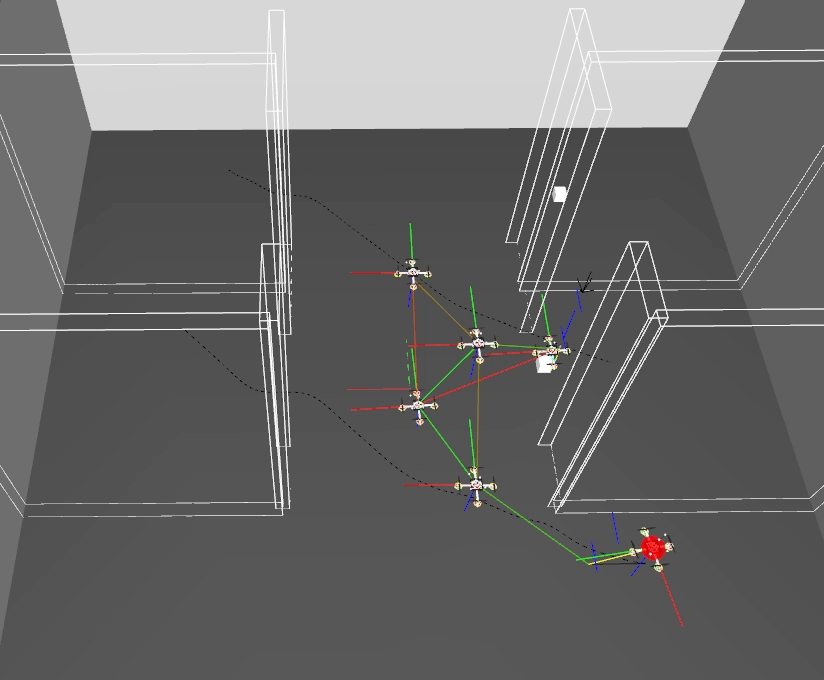}}
\caption{Three simultaneous screenshots of the experiment described in the text: \mySubref{fig:experimentScreenshots:side}~shows the side view of the scene from a fixed camera. Connections between UAVs (brightened areas) are overlayed as green lines. \mySubref{fig:experimentScreenshots:onboard}~shows the view taken from the onboard camera of the \cameraRobot{} using the same highlighting. \mySubref{fig:experimentScreenshots:SSX}~shows a 3D synthetic reconstruction of the robot positions and connections are shown with a line given in green when the weight is high, red shortly before a connection breaks and as a gradient in between. The robot that is marked with the red sphere is currently decoupled and controlled by the human operator.}
\label{fig:experimentPictures}
\end{figure*}


\begin{figure*}[t]
\begin{minipage}{0.32\textwidth}
\centering
\vspace{-0.2em}
\subfloat[\label{fig:experimentPlots:lambda2}]          {\includegraphics[width=\textwidth]{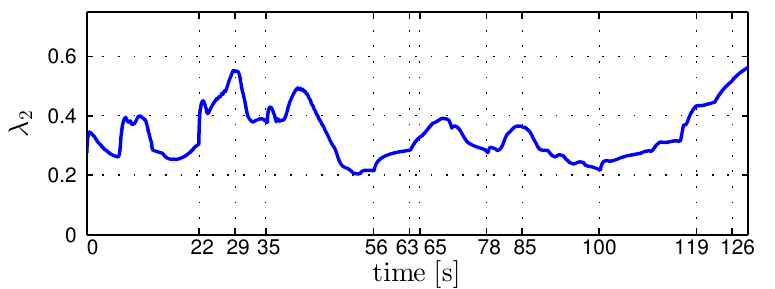}}\\
\subfloat[\label{fig:experimentPlots:numberOfLinks}]    {\includegraphics[width=\textwidth]{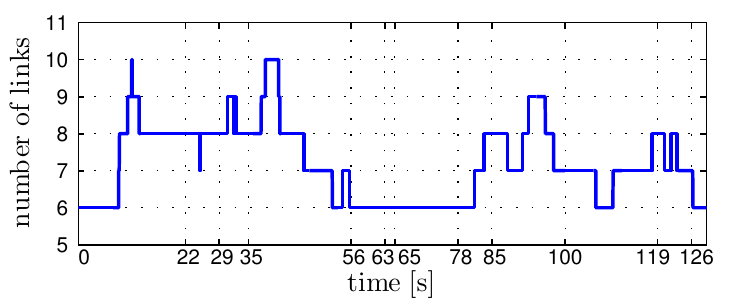}}\\
\subfloat[\label{fig:experimentPlots:stretch}]          {\includegraphics[width=\textwidth]{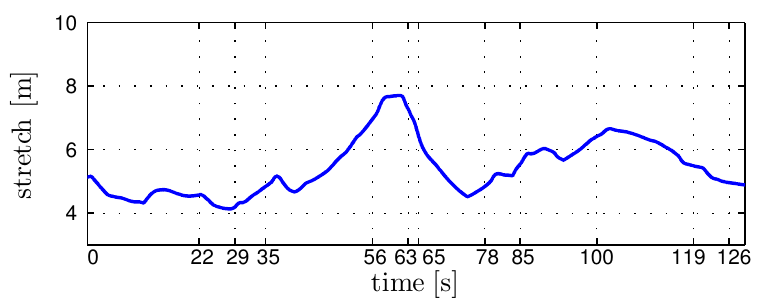}} 
\end{minipage}
\hfill
\begin{minipage}{0.32\textwidth}
\centering
\subfloat[\label{fig:experimentPlots:states}]           {\includegraphics[width=\textwidth]{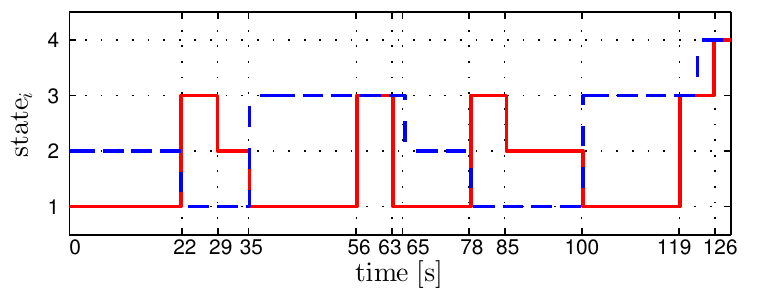}}\\
\vspace{-0.2em}
\subfloat[\label{fig:experimentPlots:posDiffAnchor}]    {\includegraphics[width=\textwidth]{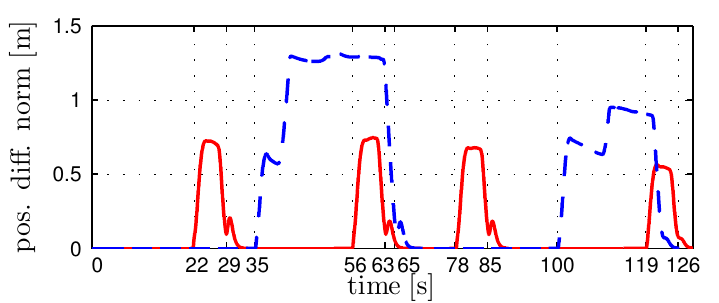}}\\
\vspace{-0.2em}
\subfloat[\label{fig:experimentPlots:Lambda_leader}]    {\includegraphics[width=\textwidth]{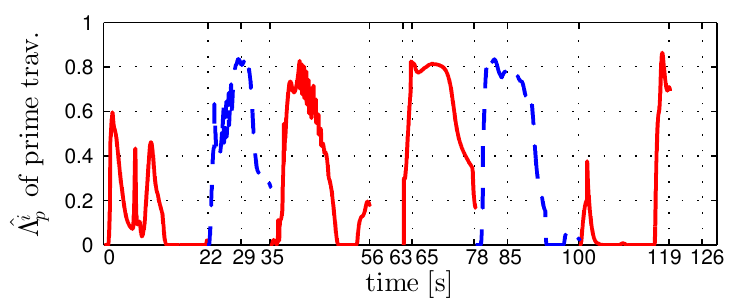}}
\end{minipage}
\hfill
\begin{minipage}{0.33\textwidth}
\centering
\subfloat[\label{fig:experimentPlots:Lambda_estimate}] {\includegraphics[width=\textwidth]{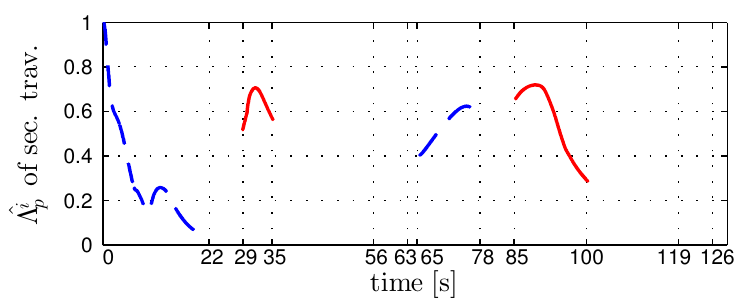}}\\
\vspace{-0.3em}
\subfloat[\label{fig:experimentPlots:Theta}]           {\includegraphics[width=\textwidth]{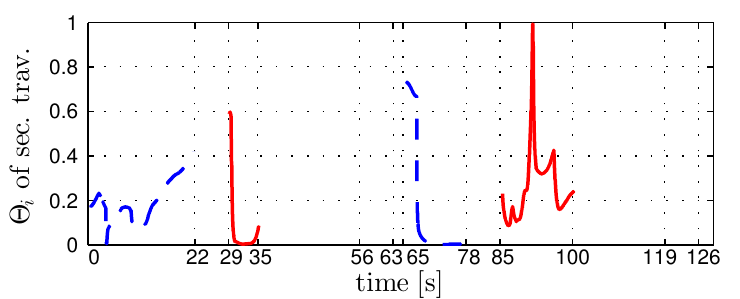}}\\
\vspace{-0.2em}
\subfloat[\label{fig:experimentPlots:rho}]             {\includegraphics[width=\textwidth]{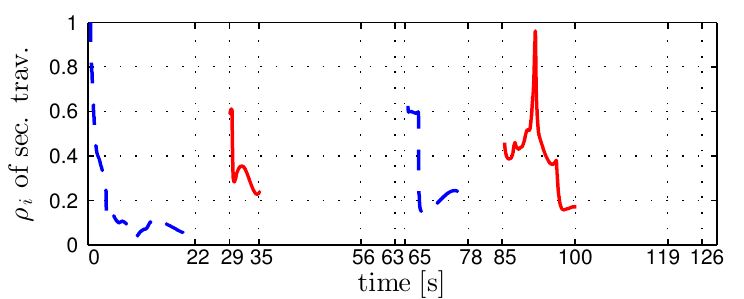}}
\end{minipage}
\caption{Behavior of different measurements during an experiment: 
\mySubref{fig:experimentPlots:lambda2}~$\lambda_2$ always keeps greater than zero, thus showing how the group remains connected at all times, 
\mySubref{fig:experimentPlots:numberOfLinks}~the number of links $|\mathcal{E}(t)|$ of the interaction graph $\mathcal{G}(t)$,
\mySubref{fig:experimentPlots:stretch}~the stretch of the formation given by the maximum Euclidean distance between any two quadrotors over time, 
\mySubref{fig:experimentPlots:states}~the exploration states with the following meaning: 1:~\primeTraveler{}, 2:~\secondaryTraveler{}, 3:~\anchor{}, 4:~\connector{},
\mySubref{fig:experimentPlots:posDiffAnchor}~the position difference between the virtual point of the connectivity maintenance and the commanded position to the quadrotors showing the decoupling as an \anchor{},
\mySubref{fig:experimentPlots:Lambda_leader}~the traveling efficiency of the current \primeTraveler{} (see~\eqref{eq:traveling_efficiency}),
\mySubref{fig:experimentPlots:Lambda_estimate}~the estimation of the traveling efficiency by the \secondaryTravelers{} (see~\eqref{eq:Lambda_estimate}),
\mySubref{fig:experimentPlots:Theta}~the force direction alignment for the \secondaryTravelers{} (see~\eqref{eq:Theta}),
\mySubref{fig:experimentPlots:rho}~the adaptive gain used by the \secondaryTravelers{} to scale down their \travelingForce{}~(see~\eqref{eq:gain}).
}
\label{fig:experimentPlots}
\end{figure*}

\Cref{fig:experimentPictures} shows three screenshots of the experiment: the lines between two quadrotors represent the corresponding connecting link as per graph $\calG$.

Finally, \cref{fig:experimentPlots} reports nine plots that capture the behavior of several quantities of interest throughout the whole experiment.  
As can be seen in \cref{fig:experimentPlots:lambda2}, the generalized algebraic connectivity eigenvalue $\lambda_2(t)$ (see Appendix~\ref{app:connectivity}) remains positive for any $t>0$, thus implying continuous connectivity of the graph $\calG$ as desired.
The time-varying number of edges in \cref{fig:experimentPlots:numberOfLinks} shows the dynamic reconfiguration of the group topology which ranges between topologies with $5$ edges (the minimum for having $\calG$ connected) and topologies with up to $10$ edges. This plot clearly shows how the adopted connectivity maintenance approach can cope with time-varying graphs. 
In~\cref{fig:experimentPlots:stretch}, we report the stretch of the group, defined as the maximum Euclidean distance between any two robots at a given time $t$. One can then appreciate how this stretch varies among $3.5$ and $7.5$ meters thus exploiting at most the allowable ranges of the experimental arena. Notice also how the stretch is in general larger when the number of links (and consequently $\lambda_2(t)$) is smaller. In fact the two peaks at about 60\,s and 103\,s occur when the robots are forced into a sparsely connected topology because the two \explorers{} have concurrently reached their farthest target pairs, i.e., ($A_1$,$B_2$) and ($B_1$,$D_2$).

\Cref{fig:experimentPlots:states} shows the \explorer{} states \varFontNoSizeSmall{state}$_2$ and \varFontNoSizeSmall{state}$_4$ over time, with a dashed blue line and solid red line, respectively. In the plot, the following code is used: $1={}$\primeTraveler{}, $2={}$\secondaryTraveler{}, $3={}$\anchor{} and $4={}$\connector{}. 
For $i=1,3,5,6$ it is \varFontNoSizeSmall{state}$_i=4$ for all $t\in[0,129]$. Notice that, because of the algorithm design, at most one \explorer{} has \varFontNoSizeSmall{state}$_i=1$ at any given time.

The temporary decoupling of the \explorers{} from the points $q_2$ and $q_4$ during their anchoring behavior can be appreciated in~\cref{fig:experimentPlots:posDiffAnchor}, where the 
Euclidian distance
between the real robot position in the trajectory and the corresponding $q_i(t)$ is shown, for $i=2,4$. \CameraRobot{} (solid red line) decouples four times in total, in correspondence of the $2$ pick-and-place operations, which gives rise to $4$ short peaks in the plot.
\PickNplaceRobot{} (dashed blue line) decouples two times in total, in correspondence of the $2$ human-in-the-loop operations, causing $2$ long peaks in the plot.

\Cref{fig:experimentPlots:Lambda_leader} shows the traveling efficiency $\Lambda_p$ of the current \primeTraveler{} with a dashed blue line when \cameraRobot{} is the \primeTraveler{} and with a solid red line when \pickNplaceRobot{} is the \primeTraveler{}. The estimation ${\hat\Lambda}_p^i$ of this value (see~\eqref{eq:Lambda_estimate}) by all robots that are currently not \primeTraveler{} is given in~\cref{fig:Lambda_propagation}. We chose $k_\Lambda=1$ resulting in a relatively slow propagation to show the additional robustness of our algorithm against this parameter (and the simple adopted consensus propagation), but clearly one could easily employ higher gains. To make it easier for the reader to understand the following discussion, we show again in~\cref{fig:experimentPlots:Lambda_estimate} the essential information of this last plot whenever a robot is a \secondaryTraveler{}. In~\cref{fig:experimentPlots:Theta,fig:experimentPlots:rho} the force direction alignment $\Theta_i$ (see~\eqref{eq:Theta}) and the adaptive gain $\rho_i$ (see~\eqref{eq:gain}) of the current \secondaryTraveler{} are shown. In the latter three plots a dashed blue line indicates when \cameraRobot{} is the \secondaryTraveler{}, and a solid red line when \pickNplaceRobot{} is the \secondaryTraveler{}. 

To fully understand the important features of our method, we now give a detailed description of the time interval $[0,22]$ in the~\cref{fig:experimentPlots:Lambda_leader,fig:experimentPlots:Lambda_estimate,fig:experimentPlots:Theta,fig:experimentPlots:rho}. A similar pattern can then be found in the rest of the experiment. In this time interval, the \pickNplaceRobot{} is the \primeTraveler{}, while \cameraRobot{} is a \secondaryTraveler{} (and the rest are \connectors{}). Due to the initial transient of its motion controller, the \primeTraveler{} starts with $\Lambda_p=0$ and quickly reaches $\Lambda_p=0.6$. Shortly after, the traveling efficiency decreases again since \pickNplaceRobot{} reaches the end of the area where it can freely move and, thus, needs to `pull' the other robots for preserving connectivity of $\calG$. This effect is propagated to the \cameraRobot{} as shown in~\cref{fig:experimentPlots:Lambda_estimate}. The force direction alignment between the \travelingForce{} and the \connectivityForce{} is shown in~\cref{fig:experimentPlots:Theta}. Combining these two plots with~\eqref{eq:gain} allows to understand the effect of~\cref{fig:experimentPlots:rho}. As can be seen, the \secondaryTraveler{} slows down its motion to around $10\%$ for roughly 5 seconds. This enables the \primeTraveler{} to travel faster again (see~\cref{fig:experimentPlots:Lambda_leader}). However, since the \pickNplaceRobot{} needs to move around the wall (see~\cref{fig:experimentPaths:1}) to reach its target, it needs to `pull' the other robots even more for preserving connectivity. Therefore, the traveling efficiency becomes zero and, although the direction alignment of the \secondaryTraveler{} becomes higher, the overall gain $\rho_i$ stays very low: this makes it possible for the \primeTraveler{} to eventually reach its target. We recall here that, according to~\cref{tab:parametersExperiments} and~\eqref{eq:velmatch}, $\Lambda_p=1$ as soon as the \primeTraveler{} achieves a speed of at least $80\%$ of its desired cruise speed (so the error is less than $20\%$), while $\Lambda_p=0$ means a speed of less than $30\%$ (an error of more than $70\%$), and \emph{not necessarily} a zero velocity.

\begin{figure}[t]
\centering
{\includegraphics[width=0.9\columnwidth]{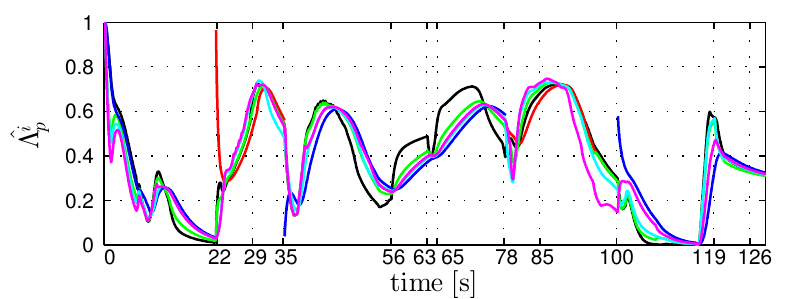}}
\caption{Estimation of the \primeTraveler{} traveling efficiency of all the six robots, whenever they are currently not a \primeTraveler{} (see~\eqref{eq:Lambda_estimate}). The color scheme for the robots is as in~\cref{fig:experimentPaths}.}
\label{fig:Lambda_propagation}
\end{figure}

\section{Conclusions}\label{sec:conclusions}

In this paper we presented a novel distributed and decentralized control strategy that enables simultaneous multi-target exploration while ensuring a time-varying connected topology in a 3D cluttered environment. We provided a detailed description of our algorithm which effectively exploits presence of \emph{four} dynamic roles for the robots in the group. In particular, a \connector{} is a robot with no active target, an \anchor{} a robot close to its desired location, and all other robots are instead moving towards their targets. Presence of at most one \primeTraveler{}, holding a leader virtue, is always guaranteed. All other robots (\secondaryTravelers{}) are bound to adapt their motion plan so as to facilitate the \primeTraveler{} visiting task. This feature ensures that the \primeTraveler{} is always able to reach its target, and thus ultimately allows to conclude completeness of the exploration strategy.
The scalability and effectiveness of the proposed method was shown by presenting a complete and extensive set of simulative results, as well as an experimental validation with real robots for further demonstrating the practical feasibility of our approach.

As future development, we plan to modify the control of the \connectors{} in order to actively improve the connectivity (e.g., moving towards the center of the group or towards the closest \explorer{}) and therefore decrease the overall completion time even more. Another extension could include imposing temporal targets that expire before any robot can possibly reach them. In our framework this could be easily achieved by letting the corresponding \primeTraveler{} or \secondaryTraveler{} switching into a \connector{} whenever a target expires, for then automatically starting to explore the next target (if any). 

An important direction worth of investigation would also be the possibility to (explicitly) deal with errors or uncertainties in the relative position measurements (w.r.t.~robots and obstacles) needed by the algorithm. Indeed, the presented results rely on an accurate measurement of robot and obstacle relative positions obtained by means of an external motion capture system. 

Another improvement could address the distributed election of the \primeTraveler{} as was already discussed in~\cref{sec:election}. Indeed, while the adopted flooding approach does not require presence of a centralized planning unit, it still needs to take into account information from all robots. It would obviously be preferable to only exploit information available to the robot itself and its $1$-hop neighbors. This could be achieved by leveraging some (suitable variant of the) consensus algorithm as done for the decentralized propagation of the traveling efficiency of the current \primeTraveler{}. 
More generally, it might also be beneficial to improve the election of the \primeTraveler{} by considering other criteria than the Euclidean distance w.r.t.~a target which may not always result in an `optimal' group motion (e.g., when obstacles, such as a wall, are present between the next \primeTraveler{} and the target).
The election could for instance choose the robot with the highest chance of reducing even further the completion time, e.g., based on the current motion of the group or direction of the majority of current targets of all \secondaryTravelers{}.

Finally, it would be interesting to obtain an analytical upper bound of the total exploration time for our approach, although, in our opinion, deriving such a bound is unfortunately not so straightforward. Clearly, the considered multi-target exploration scenario has some analogies with the multiple traveling salesman problem~\citep{2006-Bek}, where a certain number $N$ of agents are asked to find a set of $N$ shortest routes through a set of $m$ cities and return back to the start. Nevertheless, an analysis based on the multiple traveling salesman problem would not easily extend to our case because of the constraint of continuous connectivity maintenance.





\appendix

\section{Appendix}\label{app:connectivity}

For the sake of completeness and readablity, we will recap here the main features of the connectivity maintenance algorithm presented in~\cite{2013e-RobFraSecBue} with some small changes in the variable names. We start by defining $d_{ij}=\|q_i-q_j\|$ as the distance between two robot positions $q_i$ and $q_j$, and $d_{ijo}=\min_{\varsigma\in[0,1],o\in\mathcal{O}}\|q_i + \varsigma (q_j - q_i) - o\|$ as the closest distance from the line of sight between robot $i$ and $j$ to any obstacle. 
 
 The main conceptual steps behind the computation of $f_i^\lambda$ can be summarized as follows:
 \begin{enumerate} 
 \item Define an auxiliary weighted graph $\mathcal{G}^\lambda(t)=(\mathcal{V},\mathcal{E}^\lambda,W)$, where $W$ is a symmetric nonnegative $n\times n$ matrix whose entries $W_{ij}$ represent the weight of the edge $(i,j)$ and $(i,j)\in\mathcal{E}^\lambda\Leftrightarrow W_{ij}>0$. 
 \item Design every weight $W_{ij}$ as a \emph{smooth} function of the robot positions $q_i$, $q_j$ and of the obstacle points surrounding $q_i$ and $q_j$, with the property that $W_{ij}=0$ if and only if at least one of the following conditions is verified: 
 \begin{enumerate}
 \item the maximum sensing range $R_s$ is reached: $d_{ij} \ge R_s$,
 \item the minimum desired distance to obstacles $R_o$ is reached (where $R_o<R_m$): $d_{ijo} \le R_o$;
 \item the minimum desired inter-robot distance $R_c$ is reached: $d_{ik} \le R_c$ for at least one $k\neq i$.
 \end{enumerate}
 \item Compute $f_i^\lambda$ as the negative gradient of a potential function $V^\lambda(\lambda_2)$ that grows unbounded when
 $\lambda_2\to\lambda_2^\text{min}$ from above, where $\lambda_2$ is the second smallest eigenvalue of the (symmetric and positive semi-definite) Laplacian matrix
 $L=\operatorname{diag}_{i=1}^{n}(\sum_{j=1}^{n}W_{ij})-W$,
 and $\lambda_2^\text{min}$ is a non-negative parameter. This eigenvalue $\lambda_2$ is often also called Fiedler eigenvalue.
 \end{enumerate}
 
 It is known from graph theory that a graph is connected if and only if the Fiedler eigenvalue of its Laplacian is positive~\citep{1973-Fie}.
 If $\mathcal{G}^\lambda(0)$ is connected, and in particular $\lambda_2(0)>\lambda_2^\text{min}$, then under the action of $f_i^\lambda$ the
 value of $\lambda_2(t)$ can never decrease below $\lambda_2^\text{min}$ and therefore $\mathcal{G}^\lambda(t)$ always stays connected.
 
 From a formal point of view the anti-gradient of $V^\lambda$ for the $i$-th robot takes the form
 \begin{equation}\label{eq:first_step}
 f_i^\lambda=-\parder{V^\lambda(\lambda_2)}{q_i}=-\totder{V^\lambda}{\lambda_2}\parder{\lambda_2}{q_i}.
 \end{equation}
 Moreover, if the formal expression of $V^\lambda$ and $W$ are known then \eqref{eq:first_step} can be analytically computed via the expression~\citep{2010-YanFreGorLynSriSuk}, 
 \begin{equation}\label{eq:magic_formula}
 \parder{\lambda_2}{q_i}=\sum_{j\in\calN_i}\parder{W_{ij}}{q_i}(\nu_{2_i}-\nu_{2_j})^2,
 \end{equation}
 where $\nu_{2_i}$ is the $i$-th component of the normalized eigenvector of $L$ associated to $\lambda_2$.
 
 In order to have a fully decentralized computation of $f_i^\lambda$, the robots perform a distributed estimation of both $\lambda_2(t)$ and
 $\nu_{2_i}(t)$, for all $i=1,\ldots,N$, as shown in~\cite{2010-YanFreGorLynSriSuk}. In~\cite{2013e-RobFraSecBue} the authors finally prove
 the passivity (and then the stability) of the system w.r.t. the pair $(f_i,v_i)$ for all $i=1,\ldots,N$, as well as the possibility to compute the connectivity force $ f_i^\lambda$ in~(\ref{eq:first_step}) in a completely decentralized way.

\bibliography{bibAlias,bibMain,bibNew,bibAF}

\begin{thebibliography}{33}
\providecommand{\natexlab}[1]{#1}
\providecommand{\url}[1]{{#1}}
\providecommand{\urlprefix}{URL }
\expandafter\ifx\csname urlstyle\endcsname\relax
  \providecommand{\doi}[1]{DOI~\discretionary{}{}{}#1}\else
  \providecommand{\doi}{DOI~\discretionary{}{}{}\begingroup
  \urlstyle{rm}\Url}\fi
\providecommand{\eprint}[2][]{\url{#2}}

\bibitem[{Antonelli et~al(2005)Antonelli, Arrichiello, Chiaverini, and
  Setola}]{2005-AntArrChiSet}
Antonelli G, Arrichiello F, Chiaverini S, Setola R (2005) A self-configuring
  {MANET} for coverage area adaptation through kinematic control of a platoon
  of mobile robots. In: 2005 IEEE/RSJ Int. Conf. on Intelligent Robots and
  Systems, Edmonton, Canada, pp 1332--1337

\bibitem[{Antonelli et~al(2006)Antonelli, Arrichiello, Chiaverini, and
  Setola}]{2006-AntArrChiSet}
Antonelli G, Arrichiello F, Chiaverini S, Setola R (2006) Coordinated control
  of mobile antennas for ad-hoc networks in cluttered environments. In: 9th
  Int. Conf. on Intelligent Autonomous Systems, Tokyo, Japan

\bibitem[{Biagiotti and Melchiorri(2008)}]{2008-BiaMel}
Biagiotti L, Melchiorri C (2008) Trajectory Planning for Automatic Machines and
  Robots. Springer

\bibitem[{Burgard et~al(2005)Burgard, Moors, Stachniss, and
  Schneider}]{2005-BurMooStaSch}
Burgard W, Moors M, Stachniss C, Schneider F (2005) Coordinated multi-robot
  exploration. IEEE Trans on Robotics and Automation 21(3):376--386

\bibitem[{Durham et~al(2012)Durham, Franchi, and Bullo}]{2012a-DurFraBul}
Durham JW, Franchi A, Bullo F (2012) Distributed pursuit-evasion without global
  localization via local frontiers. Autonomous Robots 32(1):81--95

\bibitem[{Faigl and Hollinger(2014)}]{2014-FaiHol}
Faigl J, Hollinger GA (2014) Unifying multi-goal path planning for autonomous
  data collection. In: 2013 IEEE/RSJ Int. Conf. on Intelligent Robots and
  Systems, Chicago, IL, pp 2937--2942

\bibitem[{Fiedler(1973)}]{1973-Fie}
Fiedler M (1973) Algebraic connectivity of graphs. Czechoslovak Mathematical
  Journal 23(98):298--305

\bibitem[{Franchi et~al(2009)Franchi, Freda, Oriolo, and
  Vendittelli}]{2009b-FraFreOriVen}
Franchi A, Freda L, Oriolo G, Vendittelli M (2009) The sensor-based random
  graph method for cooperative robot exploration. IEEE/ASME Trans on
  Mechatronics 14(2):163--175

\bibitem[{Freeman et~al(2006)Freeman, Yang, and Lynch}]{2006-FreYanLyn}
Freeman RA, Yang P, Lynch KM (2006) Stability and convergence properties of
  dynamic average consensus estimators. In: 45th {IEEE} Conf. on Decision and
  Control, San Diego, CA, pp 338--343

\bibitem[{Grabe et~al(2013)Grabe, Riedel, B\"ulthoff, {Robuffo Giordano}, and
  Franchi}]{2013j-GraRieBueRobFra}
Grabe V, Riedel M, B\"ulthoff HH, {Robuffo Giordano} P, Franchi A (2013) The
  {TeleKyb} framework for a modular and extendible {ROS}-based quadrotor
  control. In: 6th European Conference on Mobile Robots, Barcelona, Spain, pp
  19--25

\bibitem[{Hollinger and Singh(2012)}]{2012-HolSin}
Hollinger G, Singh S (2012) Multirobot coordination with periodic connectivity:
  Theory and experiments. IEEE Trans on Robotics 28(4):967--973

\bibitem[{Howard et~al(2006)Howard, Parker, and Sukhatme}]{2006-HowParSuk}
Howard A, Parker LE, Sukhatme GS (2006) Experiments with a large heterogeneous
  mobile robot team: Exploration, mapping, deployment and detection. The
  International Journal of Robotics Research 25(5-6):431--447

\bibitem[{Jensfelt and Kristensen(2001)}]{2001-JenKri}
Jensfelt P, Kristensen S (2001) Active global localization for a mobile robot
  using multiple hypothesis tracking. IEEE Trans on Robotics and Automation
  17(5):748--760

\bibitem[{L\"achele et~al(2012)L\"achele, Franchi, B\"ulthoff, and {Robuffo
  Giordano}}]{2012m-LaeFraBueRob}
L\"achele J, Franchi A, B\"ulthoff HH, {Robuffo Giordano} P (2012) {SwarmSimX}:
  Real-time simulation environment for multi-robot systems. In: Noda I, Ando N,
  Brugali D, Kuffner J (eds) 3rd Int. Conf. on Simulation, Modeling, and
  Programming for Autonomous Robots, Lecture Notes in Computer Science, vol
  7628, Springer, pp 375--387

\bibitem[{Lim and Kim(2001)}]{2001-LimKim}
Lim H, Kim C (2001) Flooding in wireless ad hoc networks. Computer
  Communications 24(3-4):353--363

\bibitem[{Lynch(1997)}]{1997-Lyn}
Lynch NA (1997) Distributed Algorithms. Morgan Kaufmann

\bibitem[{Masone et~al(2012)Masone, Franchi, B\"ulthoff, and {Robuffo
  Giordano}}]{2012k-MasFraBueRob}
Masone C, Franchi A, B\"ulthoff HH, {Robuffo Giordano} P (2012) Interactive
  planning of persistent trajectories for human-assisted navigation of mobile
  robots. In: 2012 IEEE/RSJ Int. Conf. on Intelligent Robots and Systems,
  Vilamoura, Portugal, pp 2641--2648

\bibitem[{Mistler et~al(2001)Mistler, Benallegue, and
  {M'Sirdi}}]{2001-MisBenMsi}
Mistler V, Benallegue A, {M'Sirdi} NK (2001) Exact linearization and
  noninteracting control of a 4 rotors helicopter via dynamic feedback. In:
  10th {IEEE} Int. Symp. on Robots and Human Interactive Communications,
  Bordeaux, Paris, France, pp 586--593

\bibitem[{Mosteo et~al(2008)Mosteo, Montano, and Lagoudakis}]{2008-MosMonLag}
Mosteo AR, Montano L, Lagoudakis MG (2008) Guaranteed-performance multi-robot
  routing under limited communication range. In: 9th Int. Symp. on Distributed
  Autonomous Robotic Systems, Tsukuba, Japan, pp 491--502

\bibitem[{Murphy et~al(2008)Murphy, Tadokoro, Nardi, Jacoff, Fiorini, Choset,
  and Erkmen}]{2008-MurTadNarJacFioChoErk}
Murphy R, Tadokoro S, Nardi D, Jacoff A, Fiorini P, Choset H, Erkmen A (2008)
  Search and rescue robotics. In: Siciliano B, Khatib O (eds) Springer Handbook
  of Robotics, Springer, pp 1151--1173

\bibitem[{Nestmeyer et~al(2013{\natexlab{a}})Nestmeyer, {Robuffo Giordano}, and
  Franchi}]{2013l-NesRobFra}
Nestmeyer T, {Robuffo Giordano} P, Franchi A (2013{\natexlab{a}})
  Human-assisted parallel multi-target visiting in a connected topology. In:
  6th Int. Work. on Human-Friendly Robotics, Rome, Italy

\bibitem[{Nestmeyer et~al(2013{\natexlab{b}})Nestmeyer, {Robuffo Giordano}, and
  Franchi}]{2013f-NesRobFra}
Nestmeyer T, {Robuffo Giordano} P, Franchi A (2013{\natexlab{b}}) Multi-target
  simultaneous exploration with continual connectivity. In: 2nd Int. Work. on
  Crossing the Reality Gap - From Single to Multi- to Many Robot Systems, at
  2013 {IEEE} Int. Conf. on Robotics and Automation, Karlsruhe, Germany

\bibitem[{Olfati-Saber and Murray(2003)}]{2003-OlfMur}
Olfati-Saber R, Murray RM (2003) Consensus protocols for networks of dynamic
  agents. In: 2003 {A}merican {C}ontrol {C}onference, Denver, CO, pp 951--956

\bibitem[{Olfati-Saber and Murray(2004)}]{2004-OlfMur}
Olfati-Saber R, Murray RM (2004) Consensus problems in networks of agents with
  switching topology and time-delays. IEEE Trans on Automatic Control
  49(9):1520--1533

\bibitem[{Pasqualetti et~al(2012)Pasqualetti, Franchi, and
  Bullo}]{2012b-PasFraBul}
Pasqualetti F, Franchi A, Bullo F (2012) On cooperative patrolling: Optimal
  trajectories, complexity analysis, and approximation algorithms. IEEE Trans
  on Robotics 28(3):592--606

\bibitem[{Pei and Mutka(2012)}]{2012-PeiMut}
Pei Y, Mutka MW (2012) Steiner traveler: Relay deployment for remote sensing in
  heterogeneous multi-robot exploration. In: 2012 {IEEE} Int. Conf. on Robotics
  and Automation, St. Paul, MN, pp 1551--1556

\bibitem[{Pei et~al(2010)Pei, Mutka, and Xi}]{2010-PeiMutXi_}
Pei Y, Mutka MW, Xi N (2010) Coordinated multi-robot real-time exploration with
  connectivity and bandwidth awareness. In: 2010 {IEEE} Int. Conf. on Robotics
  and Automation, Anchorage, AK, pp 5460--5465

\bibitem[{{Robuffo Giordano} et~al(2013){Robuffo Giordano}, Franchi, Secchi,
  and B\"ulthoff}]{2013e-RobFraSecBue}
{Robuffo Giordano} P, Franchi A, Secchi C, B\"ulthoff HH (2013) A
  passivity-based decentralized strategy for generalized connectivity
  maintenance. The International Journal of Robotics Research 32(3):299--323

\bibitem[{Stump et~al(2008)Stump, Jadbabaie, and Kumar}]{2008-StuJadKum}
Stump E, Jadbabaie A, Kumar V (2008) Connectivity management in mobile robot
  teams. In: 2008 {IEEE} Int. Conf. on Robotics and Automation, Pasadena, CA,
  pp 1525--1530

\bibitem[{Stump et~al(2011)Stump, Michael, Kumar, and
  Isler}]{2011-StuMichKumIsl}
Stump E, Michael N, Kumar V, Isler V (2011) Visibility-based deployment of
  robot formations for communication maintenance. In: 2011 {IEEE} Int. Conf. on
  Robotics and Automation, Shanghai, China, pp 4489--4505

\bibitem[{Tardioli et~al(2010)Tardioli, Mosteo, Riazuelo, Villarroel, and
  Montano}]{2010-TarMosRiaVilMon}
Tardioli D, Mosteo AR, Riazuelo L, Villarroel JL, Montano L (2010) Enforcing
  network connectivity in robot team missions. The International Journal of
  Robotics Research 29(4):460--480

\bibitem[{Yang et~al(2010)Yang, Freeman, Gordon, Lynch, Srinivasa, and
  Sukthankar}]{2010-YanFreGorLynSriSuk}
Yang P, Freeman RA, Gordon GJ, Lynch KM, Srinivasa SS, Sukthankar R (2010)
  Decentralized estimation and control of graph connectivity for mobile sensor
  networks. Automatica 46(2):390--396

\bibitem[{Zavlanos and Pappas(2007)}]{2007-ZavPap}
Zavlanos MM, Pappas GJ (2007) Potential fields for maintaining connectivity of
  mobile networks. IEEE Trans on Robotics 23(4):812--816

\end{thebibliography}
%
\end{document}